\documentclass{article}

\PassOptionsToPackage{numbers, compress}{natbib}

\usepackage[preprint]{neurips_2026}

\usepackage[utf8]{inputenc} %
\usepackage[T1]{fontenc}    %
\usepackage{hyperref}       %
\usepackage{url}            %
\usepackage{booktabs}       %
\usepackage{amsfonts}       %
\usepackage{nicefrac}       %
\usepackage{microtype}      %
\usepackage[dvipsnames]{xcolor}         %

\usepackage{caption}
\usepackage{multirow}
\usepackage{wrapfig}
\usepackage{listings}
\usepackage{enumitem}
\usepackage{wrapfig}

\usepackage{amsmath}
\usepackage{amssymb}
\usepackage{mathtools}
\usepackage{amsthm}
\usepackage{cleveref}

\usepackage{algorithm}
\usepackage{algorithmic}

\usepackage{amsthm,amsmath,amsfonts,bm}
\usepackage{mathtools}

\newcommand{\norm}[1]{\left\|#1\right\|}
\newcommand{\inp}[2]{\langle #1, #2\rangle}

\def\Appref#1{Appendix~\ref{#1}}

\def\eqref#1{equation~\ref{#1}}

\def\1{\bm{1}}

\DeclareMathAlphabet{\mathsfit}{\encodingdefault}{\sfdefault}{m}{sl}
\SetMathAlphabet{\mathsfit}{bold}{\encodingdefault}{\sfdefault}{bx}{n}

\def\gD{{\mathcal{D}}}

\def\gN{{\mathcal{N}}}
\def\gO{{\mathcal{O}}}

\newcommand{\E}{\mathbb{E}}

\newcommand{\R}{\mathbb{R}}

\DeclareMathOperator{\Tr}{Tr}

\def\ddefloop#1{\ifx\ddefloop#1\else\ddef{#1}\expandafter\ddefloop\fi}
\def\ddef#1{\expandafter\def\csname 
bb#1\endcsname{\ensuremath{\mathbb{#1}}}}
\ddefloop ABCDEFGHIJKLMNOPQRSTUVWXYZ\ddefloop
\def\ddefloop#1{\ifx\ddefloop#1\else\ddef{#1}\expandafter\ddefloop\fi}
\def\ddef#1{\expandafter\def\csname 
b#1\endcsname{\ensuremath{\mathbf{#1}}}}
\ddefloop ABCDEFGHIJKLMNOPQRSTUVWXYZ\ddefloop
\def\ddef#1{\expandafter\def\csname 
c#1\endcsname{\ensuremath{\mathcal{#1}}}}
\ddefloop ABCDEFGHIJKLMNOPQRSTUVWXYZ\ddefloop
\def\ddef#1{\expandafter\def\csname 
h#1\endcsname{\ensuremath{\widehat{#1}}}}
\ddefloop ABCDEFGHIJKLMNOPQRSTUVWXYZ\ddefloop
\def\ddef#1{\expandafter\def\csname 
hc#1\endcsname{\ensuremath{\widehat{\mathcal{#1}}}}}
\ddefloop ABCDEFGHIJKLMNOPQRSTUVWXYZ\ddefloop
\def\ddef#1{\expandafter\def\csname 
t#1\endcsname{\ensuremath{\widetilde{#1}}}}
\ddefloop ABCDEFGHIJKLMNOPQRSTUVWXYZ\ddefloop
\def\ddef#1{\expandafter\def\csname 
tc#1\endcsname{\ensuremath{\widetilde{\mathcal{#1}}}}}
\ddefloop ABCDEFGHIJKLMNOPQRSTUVWXYZ\ddefloop

\usepackage{nicefrac}

\newsavebox\CBox

\newcommand{\diag}{\textrm{diag}}

\usepackage{booktabs}
\usepackage{subcaption}
\DeclareSymbolFont{bbold}{U}{bbold}{m}{n}
\DeclareSymbolFontAlphabet{\mathbbold}{bbold}

\newtheorem{theorem}{Theorem}
\newtheorem{lemma}[theorem]{Lemma}

\newtheorem{definition}{Definition}

\usepackage{thm-restate} %

\usepackage{soul}

\newcommand{\best}[1]{\textbf{#1}}     %
\newcommand{\secbest}[1]{\underline{#1}}  %

\newcommand{\ALGCOMMENT}[1]{\quad {\color{blue}\# #1}}
\newcommand{\ALGCOMMENTLINE}[1]{\STATE {\color{blue}\# #1}}

\title{Ori: Drift-Aware Streaming $k$-PCA for Very Low-Rank LLM Training}

\author{%
\normalfont
\begin{tabular}{@{}c@{}}
\textbf{%
Xinyu Luo\textsuperscript{1}\thanks{Equal contribution.}
\hspace{0.4em}
Jincheng Zhou\textsuperscript{1}\footnotemark[1]
\hspace{0.4em}
Mengbo Wang\textsuperscript{1,2}
\hspace{0.4em}
Zilin Shen\textsuperscript{1}
\hspace{0.4em}
Nadia Atallah Lanman\textsuperscript{2,3}
}
\\
\textbf{%
Petros Drineas\textsuperscript{1}
\hspace{0.4em}
Bruno Ribeiro\textsuperscript{1}
\hspace{0.4em}
Brian Bullins\textsuperscript{1}
}
\\[0.4em]
\texttt{\{luo466, zhou791, wang4887, shen624, natallah,}
\\
\texttt{pdrineas, ribeiro, bbullins\}@purdue.edu}
\end{tabular}%
}

\begin{document}

\maketitle
\begingroup
\renewcommand{\thefootnote}{\arabic{footnote}}

\footnotetext[1]{%
Department of Computer Science, Purdue University.}

\footnotetext[2]{%
Purdue Institute for Cancer Research, Purdue University.}

\footnotetext[3]{%
Department of Comparative Pathobiology, Purdue University.}

\endgroup

\setcounter{footnote}{3}

\begin{abstract}
A prominent family of low-rank LLM optimizers, pioneered by GaLore, performs gradient-subspace projection: at each step, they project the gradient into a low-dimensional subspace and map the update back to the full model. 
Most existing work either recomputes the gradient subspace at certain intervals, or performs post-hoc adaptation to catch up to the moving subspace.
In this work, we explicitly identify in practice the slow, stable manner in which LLM gradient subspaces drift across time.
To this end, we propose Ori,
a novel low-rank optimizer for training LLMs that captures the shifting principal gradient directions over time to achieve state-of-the-art accuracy under stringent low-rank constraints.
Ori employs a streaming $k$-PCA algorithm inspired by Oja++, whose low per-step cost enables high-frequency, on-the-fly subspace tracking.
In practice, Ori reduces the performance gap of the state-of-the-art low-rank optimizers to full-rank optimizers by up to 69.3\% on a 1B LLaMA pretraining task, and achieves the best performance on 10 out of 16 task scenarios in Roberta-Base fine-tuning. 
Together, these contributions allow Ori to further push the performance and efficiency frontiers of gradient-projecting optimizers, narrowing the gap to full-rank training for LLMs.
Code can be found at \url{https://anonymous.4open.science/r/Ori-Submission-2132}.
\end{abstract}

\section{Introduction}

Training large language models (LLMs) under strict memory budgets has motivated a variety of low-rank optimization methods that restrict either the weights or the gradients to a low-dimensional subspace~\citep{hu2022lora,zhao2024galore, mo2025parameter}. This broad field is primarily divided into two families. {\em Weight-subspace methods}, such as LoRA, adapt the model weights directly~\citep{hu2022lora, lialin2023relora}. More recently, {\em gradient-subspace methods}, pioneered by GaLore~\citep{zhao2024galore}, project the gradients into a low-rank subspace before the optimizer step. The latter approach and its subsequent works have proven particularly effective, demonstrating superior memory savings and performance by focusing on the dynamics of the optimizer itself~\citep{zhao2024galore, zhang2024q, lialin2023relora, han2024sltrain, dettmers2023qlora, robert2024ldadam}.

Yet most gradient-subspace methods in this promising family face a core technical challenge: that the principal gradient directions, instead of being static, evolve as the model weights are updated throughout training.
Consequently, the success of any such method hinges on how effectively it adapts to this changing subspace. 
GaLore tackles this by periodically recomputing the entire subspace from scratch via an expensive SVD~\citep{zhao2024galore}, a process that acts as a ``cold restart'' and discards prior subspace information. More recent methods, such as Online Subspace Descent (OSD)~\citep{liang2024memory}, move towards more continuous adaptation by using a secondary, non-convex optimization loop to ``catch up'' to the moving subspace.

\begin{wrapfigure}{r}{0.48\textwidth}
    \centering
    \includegraphics[width=0.40\textwidth]{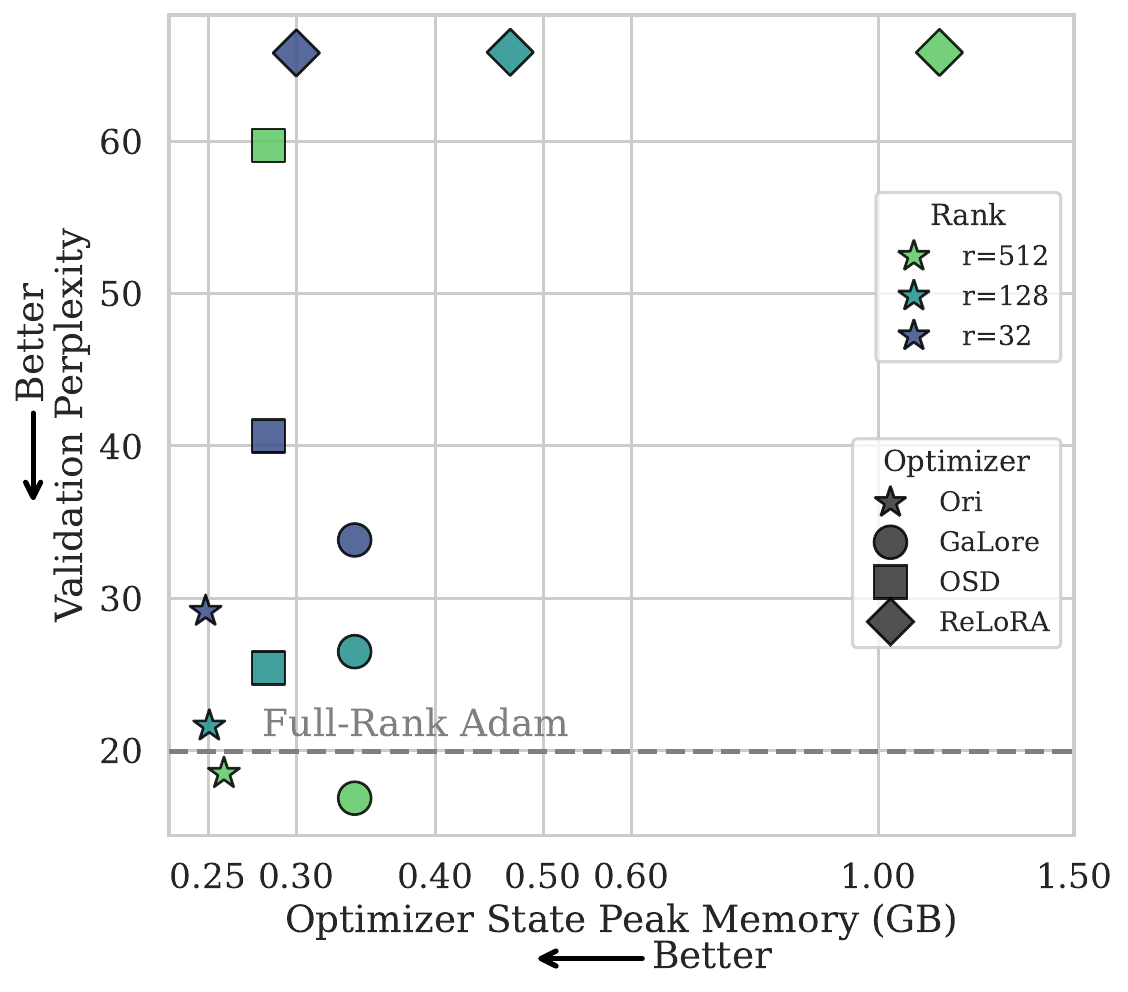}
    \caption{
        Trade-off between evaluation perplexity (Y-axis, lower is better) and peak optimizer state memory (X-axis, lower is better) for various low-rank optimizers in different ranks on the 1B LLaMA pre-training task. Bottom-left corner is the ideal region. 
        \textbf{Ori ($\star$) forms a better Pareto frontier, achieving a superior performance-to-memory trade-off compared to state-of-the-art baselines like GaLore ($\circ$) and OSD ($\square$).}
    }
    \label{fig:pretrain-frontier}
    \vspace{-25pt} 
\end{wrapfigure}

However, existing adaptation strategies may overlook a key empirical observation that we make explicit in this work: the gradient subspace does not jump erratically but, rather, drifts smoothly after warm-up. 
In other words, the projector at step $t$ is predictably close to that at $t-\Delta t$ for small $\Delta t$
(e.g., 20).
This slow, structured drift suggests a different perspective. Rather than recomputing the subspace (as in GaLore~\citep{zhao2024galore}) or optimizing it with non-convex gradient descent (as in OSD~\citep{liang2024memory}), we can track it with a streaming update that makes small, frequent rotations aligned with the most recent gradients, while 
also
preserving and transferring optimizer statistics across these rotations.
In light of these observations, we arrive at a different design problem:
\begin{center}
\textit{How do we continuously maintain a memory-light, compute-efficient, and stable projector that tracks \\ a drifting gradient subspace?}
\end{center}

We thus introduce \textbf{Ori}, a streaming $k$-PCA-based \textbf{O}ja-style updated low-\textbf{r}ank \textbf{i}terative optimizer,  
designed to exploit the slow-drift regime of LLM gradient subspaces. Rather than treating each subspace estimation as an independent, ``cold-start'' problem---as is the case with periodic SVD recomputation in GaLore---Ori treats subspace estimation as a continuous tracking problem.
The core innovation of Ori lies in its drift-aware bootstrapping mechanism: we carry over (transfer) the subspace basis from previous steps as the initialization state for the current step. Because of the slow-drift phenomenon, the prior basis provides a suitable starting point that requires only a few lightweight, streaming $k$-PCA-inspired updates to realign with the instantaneous gradient geometry. 

In practice, Ori instantiates this via the Oja++ algorithm~\citep{oja1982simplified,allen2017first}, which enables high-frequency basis rotations at a per-step cost of only $\mathcal{O}(dk^2)$. This approach allows for a warm-start transfer of Adam-style first and second moments between successive bases, ensuring stable optimization dynamics. Beyond computational efficiency, we show that this formulation allows Ori to achieve provable sublinear dynamic regret bounds under drift, providing new theoretical tools and guarantees for tracking fidelity that static recomputation methods lack.

Empirically, Ori narrowly tracks the evolving gradient subspace and converts this tracking fidelity into end-to-end wins where it matters most: very low ranks. Under tight rank budgets, each dimension in the projector is precious, and stale misalignment is especially costly. 
Empirical results show that Ori reduces the performance gap of state-of-the-art low-rank baselines to the full-rank Adam optimizer by \textbf{up to 69.3\%} on the 1B LLaMA pre-training task, where the optimizer state rank is 128 in contrast to the model embedding dimension of 2048. And on Roberta-Base fine-tuning tasks, where the ranks are even lower (8 or 4), Ori achieves superior performance on \textbf{10 out of 16} scenarios.

Moreover, Ori achieved this superior performance with one of the lowest computation budgets. 
See Figure~\ref{fig:pretrain-frontier} for a performance–memory overview on LLaMA-1B pre-training, where Ori consistently sits on the bottom-left frontier.
Through ablation, we observe that (i) increasing the projector-update period degrades performance—quantifying the cost of staleness, and (ii) injecting a small carry-over (via a drift coefficient $\rho\approx0.99$) improves accuracy—confirming the value of continuity and statistic transfer. 
Across pre-training and fine-tuning settings, Ori consistently improves the performance–memory trade-off over LoRA, GaLore, and OSD in the low-rank regime, thereby reducing the gap to full-rank Adam. 
\textbf{Our primary contributions in this work are as follows:}

\begin{itemize}[nosep,leftmargin=*]
    \item We identify and quantify slow gradient subspace drift during LLM training and argue that exploiting its continuity -- not merely reacting to it -- matters for low-rank optimization.
    \item We propose Ori, a drift-aware low-rank optimizer that equips gap-free streaming $k$-PCA with moment transfer to achieve continuous, high-frequency projector updates at negligible cost.
    \item We provide theoretical support via new dynamic-regret bounds under drift, yielding sublinear regret in the noiseless setting with mild dependence on the consecutive drift parameter.
    \item We present empirical results showing Ori’s superior performance under stringent rank constraints, with ablations isolating the roles of update frequency and drift-aware continuity.
\end{itemize}

\section{Background and Related Work}
We refer the reader to \Appref{app:more-related-work} for a more detailed discussion of related work.

\vspace{-5pt}
\paragraph{The Landscape: Low-Rank Optimization.}
Low-rank optimization in LLMs can be broadly divided into two categories. \textbf{Weight-subspace methods} adapt the model weights by constraining them to a low-rank manifold, a paradigm popularized by LoRA and its many variants~\citep{hu2022lora, lialin2023relora, zhang2023adalora, mo2025parameter, han2024sltrain}. In contrast, Ori belongs to the second category, \textbf{gradient-subspace methods}, which operate on a low-rank projection of the gradients. This line of work was pioneered by GaLore~\citep{zhao2024galore}, which uses periodic Singular Value Decomposition (SVD) to identify and project gradients onto a low-rank basis. Many subsequent works have built upon this foundation by incorporating diverse strategies such as quantization~\citep{zhang2024q}, structured sparsity~\citep{muhamed2024grass}, random projections~\citep{he2024subspace, wang2024flora}, and importance sampling~\citep{zhang2025i3s, liao2024galore}. However, a common limitation underlines these approaches: by relying on static or infrequently updated projectors, they are inherently reactive to the evolving gradient landscape. This can lead to a ``staleness penalty,'' where gradients are projected onto an outdated basis, losing critical information for optimization. Ori is designed to directly counteract this penalty by replacing infrequent recomputation with a continuous, online tracking mechanism.

\vspace{-5pt}
\paragraph{The Mechanism: Online Subspace Tracking.}
The challenge of staleness in subspace projection has motivated online methods that update projectors frequently without relying on full SVD. Recent work explores {\bf auxiliary optimization techniques} for dynamic subspace learning~\citep{liang2024memory, nguyen2024lean, robert2024ldadam, rajabi2025subtrack++}. \citet{nguyen2024lean} propose Subset-Norm and Subspace-Momentum, which decouples optimization by applying momentum in a low-dimensional subspace while performing SGD in the orthogonal complement. \citet{robert2024ldadam} use block power iteration on a weighted gradient-EMA combination to adapt optimizer states to evolving subspaces. \citet{rajabi2025subtrack++} track subspaces on the Grassmannian manifold via geodesic updates, leveraging past estimates and error feedback for accuracy.
~\citet{liang2024memory} propose a robust procedure using a secondary optimizer to perform gradient descent directly on the subspace manifold, initialized from the previous projector.

Ori adopts a fundamentally distinct strategy. Rather than relying on a secondary optimizer, it leverages the streaming $k$-PCA algorithm, \textbf{Oja++}~\citep{allen2017first}, to update the projector. This distinction carries significant implications. First, gradient descent on the subspace manifold is inherently non-convex and complex, but Oja-style multiplicative updates come with strong theoretical assurances. Specifically, Oja++ achieves a provably ``gap-free'' convergence rate, offering greater robustness in static settings compared to gradient-based methods. Second, as formalized in our theoretical contributions, our Oja-based update exhibits desirable properties even in non-stationary, drifting environments, yielding strong dynamic regret bounds. Finally, the Oja++ update is computationally cheaper than gradient descent on the manifold, enabling Ori to update the projector more frequently. This higher frequency allows for improved tracking of the evolving gradient landscape.

\section{Ori Optimizer}

In this section, we further discuss the technical foundation of Ori. First, we establish and quantify the problem of gradient subspace drift, highlighting the limitations of existing approaches. Next, we introduce the core motivation of bootstrapping subspace updates, which forms the conceptual basis of our method. We then detail the specific online algorithm, a modified Oja++, used to implement this solution. Finally, we present the complete Ori algorithm, synthesizing these components into a cohesive and efficient optimization procedure.

\subsection{The Phenomenon of Gradient Subspace Drift}

\paragraph{Formalizing and Quantifying Drift.}
Low-rank optimizers that rely on periodic updates, such as GaLore~\citep{zhao2024galore}, implicitly assume that the principal gradient subspace remains relatively stable between updates. In this work, we propose to explicitly characterize this evolution. 
To do so, we adopt a formal metric from prior work on the changing geometry of optimization trajectories: the \emph{canonical distance measure}~\citep{gur2018gradient}. 
At any training step $t$, the top-$k$ principal gradient directions $U^{(t)}_k$, which can be found via SVD of the gradient matrix $G_t = U^{(t)}\Sigma {V^{(t)}}^\top$, span a subspace represented by the orthogonal projector $P_t = U^{(t)}_k {U^{(t)}_k}^\top$.
Our notion of gradient subspace overlap, based on the canonical distance between the gradient subspace at step $t$ and a later step $t+\Delta t$, is then defined as 
\begin{align}  \label{eq:subspace-overlap}
    \text{overlap}\left(U^{(t)}_k, U^{(t+\Delta t)}_k\right) \equiv \frac{\text{Tr}(P_tP_{t+\Delta t})}{\sqrt{\text{Tr}(P_t)\text{Tr}(P_{t + \Delta t})}}
    = 1 - \frac{\norm{P_t - P_{t+\Delta t}}_F^2}{2k} 
    \propto \norm{P_t - P_{t +\Delta t}}_F^2.
\end{align}
This metric,\footnote{
\Cref{eq:subspace-overlap} is well-defined because the projector $P$ is idempotent ($P^2 = P$) and its trace equals its rank $\text{Tr}(P) = k$; we have $\norm{P_t - P_{t + \Delta t}}_F^2 = \text{Tr}(P_t) + \text{Tr}(P_{t+\Delta t}) - 2\text{Tr}(P_t P_{t +\Delta t}) = 2k - 2\text{Tr}(P_t P_{t +\Delta t})$.
}
which ranges from 0 (orthogonal) to 1 (identical), is directly related to the geometric overlap between two subspaces. Specifically, $\text{Tr}(P_t P_{t+\Delta t})$ captures the $k$ principal angles, $\{\theta_1,\dots\theta_k\}$, between subspaces $P_t$ and $P_{t+\Delta t}$, in that
\vspace{-5pt}
\begin{align*}
    \text{Tr}(P_tP_{t+\Delta t}) = \sum_{i = 1}^k \cos^2{\theta_i} \leq k.
\end{align*}
\vspace{-5pt}
Thus, it provides a principled tool for empirically measuring the rate and nature of subspace evolution.

\paragraph{Empirical Evidence of Slow Drift.}

\begin{wrapfigure}{r}{0.5\textwidth}
    \vspace{-10pt} 
    \centering
    \includegraphics[width=0.5\textwidth]{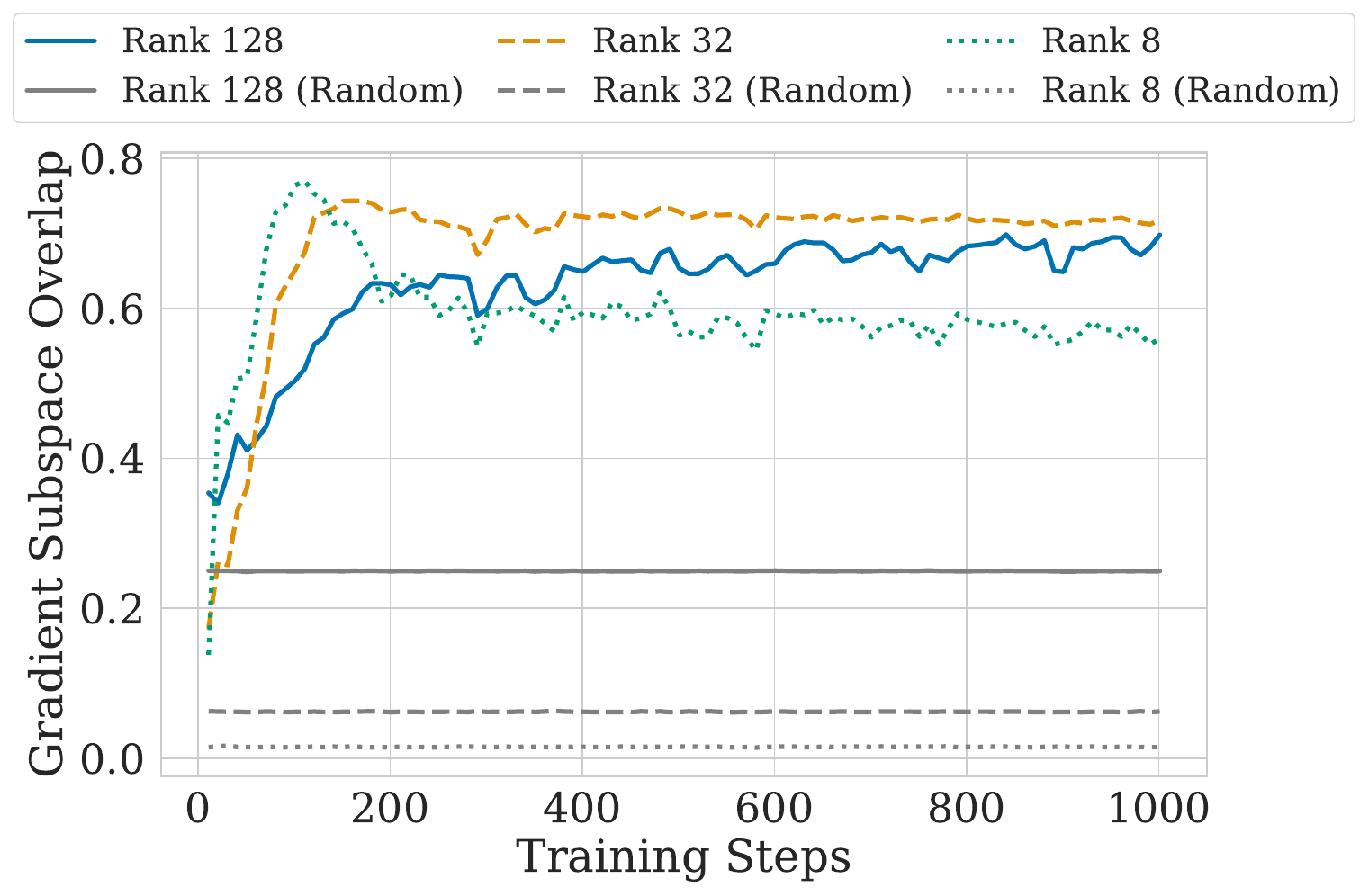}
    \caption{
        \textbf{Average gradient subspace overlap during pre-training.} We plot the overlap (\Cref{eq:subspace-overlap}) between subspaces at step $t$ and $t+\Delta t$ (with $\Delta t=10$), averaged over a 60M LLaMA model. The high, stable overlap, significantly above a random baseline, confirms that the subspace evolves gradually.
    }
    \label{fig:subspace-overlap}
\end{wrapfigure}

With the subspace overlap metric established, we empirically tested the stability assumption by measuring the overlap (with $\Delta t = 10$) throughout the pre-training of a 60M LLaMA model via the Adam optimizer on the C4 dataset
(further experiment details can be found in \Appref{app:exp-details}).
The results, shown in Figure~\ref{fig:subspace-overlap}, provide clear evidence for a ``slow drift'' hypothesis. After an initial transient phase, the average gradient subspace overlap across all tested ranks ($k \in \{8, 32, 128\}$) stabilizes at a high value between 0.55 and 0.75. Crucially, this is significantly higher than the overlap observed for random subspaces of the same rank, confirming the existence of meaningful temporal structure. This indicates: (i) \textbf{The static subspace assumption is false}, as the overlap is consistently below 1.0, and the change over just 10 steps is substantial. (ii) \textbf{The evolution is not erratic}, as the overlap remains remarkably stable, suggesting the subspace moves with a predictable ``velocity'' rather than teleporting randomly. (iii) \textbf{The phenomenon is fundamental}, as the stable drift is present across a range of ranks, suggesting it is an intrinsic property of the optimization trajectory. This temporal coherence is precisely the property that an online tracking algorithm can exploit.

\subsection{Online Subspace Tracking via Oja++ with Moment Transfer}
\begin{wrapfloat}{algorithm}{R}{0.52\textwidth}
\centering
\small
\captionof{algorithm}{Ori Projector Update (Oja++)}
\label{alg:update_projector}
\begin{algorithmic}[1]
\REQUIRE $Q \in \mathbb{R}^{d \times r}$, gradient $G \in \mathbb{R}^{d \times n}$,
$H_{\text{oja}}$: $N_{\text{updates}}, \rho, N_{\text{vecs}}$, LR schedule for $\eta$

\STATE \textbf{Function} \textsc{Update-Projector}$(Q,G,H_{\text{oja}})$
\STATE $Q_{\text{init}} \leftarrow \rho Q + (1-\rho)\mathrm{Randn}(d,r)$ \ALGCOMMENT{\Cref{eq:bootstrap}}
\STATE $Q_{\text{new}} \leftarrow \mathbf{0}_{d\times r}$
\STATE $r_{\text{curr}} \leftarrow 0$

\ALGCOMMENTLINE{Oja++ main update loop \citep{allen2017first}}
\FOR{$k = 0$ to $\log_2(r)-1$}
  \STATE $b \leftarrow r / 2^{k+1}$ \ALGCOMMENT{One block of Oja++ updates}
  \STATE $r_{\text{next}} \leftarrow r_{\text{curr}} + b$
  \STATE $Q_{\text{new}}(:,r_{\text{curr}}{:}r_{\text{next}})
  \leftarrow Q_{\text{init}}(:,r_{\text{curr}}{:}r_{\text{next}})$
  \STATE $Q_{\text{slice}} \leftarrow Q_{\text{new}}(:,0{:}r_{\text{next}})$
  \ALGCOMMENT{Current active slice}
  \STATE $Q_{\text{slice}} \leftarrow
  \textsc{Oja-Step}(Q_{\text{slice}},G,N_{\text{updates}},N_{\text{vecs}})$
  \STATE $Q_{\text{new}}(:,0{:}r_{\text{next}}) \leftarrow Q_{\text{slice}}$
  \STATE $r_{\text{curr}} \leftarrow r_{\text{next}}$
\ENDFOR

\STATE \textbf{return} $Q_{\text{new}}$
\STATE \textbf{End Function}

\ALGCOMMENTLINE{Subroutine: Oja updates, main update loop}
\STATE \textbf{Function} \textsc{Oja-Step}$(Q_{\text{in}},G,N_{\text{updates}},N_{\text{vecs}})$
\FOR{$j = 1$ to $N_{\text{updates}}$}
  \STATE $\eta \leftarrow \textsc{GetCosineAnnealingLR}(j)$
  \STATE $G_{\text{sub}} \leftarrow \textsc{SampleColumns}(G,N_{\text{vecs}})$
  \STATE $Q_{\text{in}} \leftarrow Q_{\text{in}}
  + \eta G_{\text{sub}}(G_{\text{sub}}^\top Q_{\text{in}})$ \ALGCOMMENT{\Cref{eq:oja-update}}
\ENDFOR
\STATE \textbf{return} $Q_{\text{in}}$
\STATE \textbf{End Function}
\end{algorithmic}
\end{wrapfloat}

The empirical validation of slow, structured subspace drift motivates a paradigm shift from periodic, expensive \textit{recomputation or adaption}, to continuous, lightweight \textit{online tracking}, leading naturally to the use of streaming $k$-PCA algorithms, for which Oja's method provides a classic foundation~\citep{oja1982simplified}. Its core is a simple multiplicative update that incrementally rotates the current basis $Q_{t-1}$ with a new data sample $x_t$ (in our case, derived from the gradient matrix $G_t$) to better align with new information in $x_t$, followed by a re-orthogonalization step (typically a QR decomposition) to maintain the basis properties. The fundamental update equation is:
\begin{equation} \label{eq:oja-update}
    Q_t \leftarrow \text{QR}\left( (I + \eta_t x_t x_t^\top) Q_{t-1} \right),
\end{equation}
where $\eta_t$ is a learning rate (typically controlled by a cosine scheduler) and QR denotes the QR decomposition used for re-orthogonalization. While standard Oja's algorithm initializes the process with a random matrix, our key innovation is to explicitly leverage the temporal coherence of the subspace drift. We introduce a bootstrapping mechanism that uses the final projector from the previous update step, $Q_{t-1}$, as a warm start for the current one. This is controlled by a drift coefficient $\rho$, which interpolates between the previous subspace and a new random Gaussian matrix $R\sim \gN(0, 1)$:
\begin{equation} \label{eq:bootstrap}
    Q_{\text{init}} = \rho Q_{t-1} + (1 - \rho) R.
\end{equation}
This matrix $Q_{\text{init}}$ then serves as the next starting point for the update rule in Equation~\ref{eq:oja-update}. 

While \Cref{eq:oja-update} and \Cref{eq:bootstrap} form the conceptual basis, we specifically employ a more robust variant, Oja++~\citep{allen2017first}, for its crucial theoretical advantage of \textbf{gap-free convergence}.
This means its convergence rate is independent of the eigengap (the difference between consecutive singular values of the gradient covariance matrix), a property vital for LLM training where such gaps can be small. The bootstrapping technique applies directly to Oja++ as well: the previous basis $Q_{t-1}$ serves as the warm start for the multi-stage Oja++ update. This combination of a theoretically sound, gap-free algorithm with a drift-aware bootstrapping strategy gives rise to the full form of our proposed optimizer, Ori, whose core projector update routine is detailed in \Cref{alg:update_projector}. For the algorithm pseudocode of Ori with AdamW,
we refer the readers to \Cref{alg:main_loop_detailed} in \Appref{app:algorithm-details}.
In particular, we provide an amortized runtime complexity analysis of Ori in \Appref{app:complexity}.

\paragraph{Applying Ori on Other Optimizers.}
To demonstrate that the design principles behind Ori can be applied to specialized, state-of-the-art low-rank optimizers other than GaLore, we examine LDAdam~\citep{robert2024ldadam}. LDAdam enhances the standard Adam optimizer by employing a projection-aware, momentum-based update alongside a generalized error-feedback mechanism. This design aligns with findings from MicroAdam~\citep{modoranu2024microadam}, which showed that applying error feedback at every iteration is essential for recovering Adam-like convergence at million- and billion-parameter scales (e.g., BERT/LLaMA). These techniques significantly enhance language model pre-training performance in extremely low-rank settings (e.g., rank 32 with an embedding size of 2048). However, LDAdam faces a critical barrier to practical deployment: it requires updating the projector at every training step ($T=1$) to maintain performance, incurring prohibitive computational costs. As shown in our later empirical results (\Cref{tab:pretrain-c4-ldadam}), when the update interval is increased to $T=200$ -- matching the standard schedule of GaLore -- LDAdam’s performance degrades significantly. We attribute this drop to LDAdam's reliance on a simple single-step power iteration method for subspace estimation, which we hypothesize is insufficient for tracking the drifting gradient subspace over large update intervals. Inspired by Ori, we introduce a multi-step stochastic power iteration mechanism into LDAdam’s update routine. We refer to this improved optimizer as O-LDAdam, reflecting its infusion of Ori’s online subspace tracking into LDAdam’s error-correcting framework. We provide further algorithmic details in \Appref{app:oLDAdam}.

\section{Dynamic Regret Analysis with Oja}\label{sec:dyn-regret-oja}

Our theoretical analysis departs from the classical static view of streaming $k$-PCA: rather than asking whether an online method converges to a fixed principal subspace, we ask whether it can track a drifting low-rank comparator under the notion of concept drift~\cite{crammer2010regret}. To this end, we aim to establish \textit{dynamic regret bounds}.

To formalize this question, we consider a nonstationary environment segmented into $S$ steps, each comprising $T$ sequential updates. Let $x_\zeta\in\R^d$ denote the streaming data at unified time index $\zeta\in[ST]$, where $\zeta=(s-1)T+t$ refers to update $t\in[T]$ inside step $s\in[S]$. We assume each $x_\zeta$ is drawn i.i.d from an unknown underlying distribution $\gD$. Let the step-$s$ optimal $k$-subspace be spanned by $V_s\in\R^{d\times k}$ with orthonormal columns. The dynamic comparator is a sequence of rank-$k$ orthogonal projectors $\{P_s\}_{s=1}^S$ which is $P_s=V_sV_s^\top$ with \emph{bounded drift} across consecutive steps:
\[
  \norm{P_{s}-P_{s+1}}_F \;\le\; \lambda,
  \qquad s\in[S-1].
\]
We now instantiate the analysis for multiplicative Oja updates on an explicit basis. The learner maintains an orthonormal basis $Q^{(\zeta)}\in\R^{d\times k}$ and updates with a noisy perturbation followed by reorthonormalization:
\begin{align}
  Q^{(\zeta+1)} \;=\; \mathrm{QR}\!\left(\big(I+\eta\,\hat g_\zeta\big)\,Q^{(\zeta)}\right),
  \qquad \hat g_\zeta \;=\; \hat x_\zeta \hat x_\zeta^\top, 
  \qquad \hat x_\zeta \;=\; x_\zeta + y_\zeta,
  \label{eq:oja}
\end{align}
where $y_\zeta$ is an additive (possibly adversarial) noise term. 
At each $\zeta$, the learner outputs a rank-$k$ projector $P^{(\zeta)}$, an orthonormal basis $Q^{(\zeta)}\in\R^{d\times k}$ with $P^{(\zeta)}=Q^{(\zeta)}(Q^{(\zeta)})^\top$. As in~\citep{marinov2018streaming}, we measure performance by the (variance) regret against the stepwise comparator:
\begin{align}
  R_S
  \;:=\;
  \sum_{\zeta=1}^{ST}\!\Big(x_\zeta^\top P_{s(\zeta)} x_\zeta - x_\zeta^\top P^{(\zeta)} x_\zeta\Big)  
  \;=\;
  \sum_{\zeta=1}^{ST}\!\inp{P_{s(\zeta)}-P^{(\zeta)}}{x_\zeta x_\zeta^\top},
  \label{eq:dyn-regret}
\end{align}
where $s(\zeta)=\lceil \zeta/T\rceil$. Since $\norm{x}_2^2 = \norm{x-Px}_2^2 + x^\top P x$ for any projector $P$, maximizing captured variance is equivalent to minimizing squared residual---so \Cref{eq:dyn-regret} is also the residual regret.

Establishing this measure of performance leads us to the following guarantee for bounding the residual regret in this setting, the proof of which is found in~\Appref{app:theo2}.

\begin{theorem}[Noisy Oja under $\lambda$-drift]\label{thm:oja}
For $\zeta \in [ST]$, let $Q^{(\zeta)}$ follow \Cref{eq:oja} with $\eta = \sqrt{\frac{\frac{k}{2}+(S-1)\sqrt{k}\lambda}{kd\cdot ST}}$, let
    $E_{\mathrm{tot}}:=\sum_{\zeta=1}^{ST} \|E_\zeta\|_2$ where $E_\zeta:=x_\zeta y_\zeta^\top+y_\zeta x_\zeta^\top+y_\zeta y_\zeta^\top$,
and suppose $\norm{x_\zeta}_2\le 1$. Then we have
\begin{equation}
  \E[R_S]
  \le
  2\sqrt{kd\cdot ST\!\left(\tfrac{k}{2} + (S-1)\sqrt{k}\,\lambda\right)} + 2k\,E_{\mathrm{tot}}.
  \label{eq:ojs-main-min}
\end{equation}
\end{theorem}

We note in the noiseless setting, i.e., $E_{\mathrm{tot}} = 0$, we achieve sublinear regret in terms of the total $ST$ updates. Additionally, our results establish a natural balance in the dependence between the number of steps and sequential updates, with the dependence of the regret on $S$ degrading as $\lambda$ increases.

\section{Experiments}

We aim to demonstrate that Ori: (i) achieves superior performance against existing low-rank optimizers in both pre-training and fine-tuning, particularly under extremely low rank scenarios; (ii) improves memory and computational efficiency; and (iii) derives its benefits from its core design principles, as shown through ablation studies.

\subsection{Baselines and Experimental Setup}  \label{sec:exp-baseline}

We evaluate Ori against a consistent set of state-of-the-art optimizers, all built upon the Adam optimizer. For each baseline, we adopt the hyperparameter configurations recommended in their respective original publications and official codebases. Full experimental details are provided in \Appref{app:exp-details}. Our baselines include: \textbf{Full-Rank Adam}, the standard optimizer serving as a practical upper bound on performance (using default $\beta_1=0.9, \beta_2=0.999$, and $\epsilon=10^{-8}$); \textbf{Weight-Subspace Methods}, represented by \textbf{LoRA}~\citep{hu2022lora} for fine-tuning (with a LoRA alpha of 32 and a dropout of 0.05) and its pre-training variant \textbf{ReLoRA}~\citep{lialin2023relora} (with optimizer states reset every 1000 steps and no full-rank warmup); and \textbf{Gradient-Subspace Methods}. This category includes \textbf{GaLore}~\citep{zhao2024galore}, which employs periodic SVD (update frequency $T=200$; pre-training scale $\alpha=0.25$; fine-tuning scale $\alpha \in \{2, 4\}$), and \textbf{OSD (Online Subspace Descent)}~\citep{liang2024memory}, another state-of-the-art method that performs gradient descent directly on the subspace manifold (similarly using $T=200$, pre-training $\alpha=0.25$, and fine-tuning $\alpha \in \{2, 4\}$).
Finally, \textbf{LDAdam}~\citep{robert2024ldadam}, which update its projection matrix via power iteration, alongside orthogonal techniques such as an error feedback mechanism and momentum-based state updates. Although the original LDAdam was designed exclusively for most frequent per-step updates ($T=1$), for the purpose of this work, we adapted it also to the efficient delayed update setting ($T=200$). For LDAdam's other hyperparameters, we use the values suggested by the authors 
(interpolation factor $\rho=0.908$, first step projection matrix initialized via SVD).

\begin{table*}[t]\small
\centering
\caption{Comparison with full- and low-rank optimizers on pre-training various sizes of LLaMA models on C4 dataset. Validation perplexity is reported (lower $\downarrow$ is better). $k$ is the rank for low-rank optimizers, and $d_{\text{model}}$ is the model embedding dimension.
We show the best scores in bold and second best ones with underlines.
\textbf{Ori is on-par with the best baseline on moderately high rank settings, and achieves superior performance when the rank $k$ is severely constrained.}
}
\label{tab:pretrain-c4}
\resizebox{\textwidth}{!}{%
\begin{tabular}{l rrr rrr rrr rrr}
\toprule
\multicolumn{1}{c}{\multirow{2}{*}{\textbf{Method}}} 
& \multicolumn{3}{c}{\textbf{60M} ($d_{\text{model}}=256$)} & \multicolumn{3}{c}{\textbf{130M} ($d_{\text{model}}=768$)} & \multicolumn{3}{c}{\textbf{350M} ($d_{\text{model}}=1024$)} & \multicolumn{3}{c}{\textbf{1B} ($d_{\text{model}}=2048$)} \\
\cmidrule(lr){2-4} \cmidrule(lr){5-7} \cmidrule(lr){8-10} \cmidrule(lr){11-13}
& $r=128$ & $r=32$ & $r=8$ & $r=256$ & $r=64$ & $r=16$ & $r=256$ & $r=64$ & $r=16$ & $r=512$ & $r=128$ & $r=32$ \\
\midrule
\multicolumn{13}{c}{\textit{\underline{Full-Rank Methods}}} \\
\addlinespace %
Adam                & & 39.03 & & & 27.60 & & & 20.68 & & & 19.92 \\
\midrule
\multicolumn{13}{c}{\textit{\underline{Low-Rank Methods}}} \\
\addlinespace %
ReLoRA~\citep{lialin2023relora}     & 144.48    & 103.96    & 103.99    & 85.53     & 85.54     & 85.54   & 72.59  & 72.72  & 72.69  & 65.82 & 65.83 & 65.79 \\
GaLore~\citep{zhao2024galore}      & \best{39.38} & 53.27  & 75.65     & \best{27.70} & \secbest{36.82}  & 60.29   & \best{21.26}  & \secbest{26.31}  & \secbest{41.25} & \best{16.87} & 26.49  & \secbest{33.82} \\
OSD~\citep{liang2024memory}          & 49.64     & \secbest{52.35} & \secbest{74.08} & 49.57   & 37.56     & \secbest{54.05}   & 28.28  & 37.13  & 50.08 & 59.72 & \secbest{25.43} & 40.63 \\
\midrule
\textbf{Ori (Ours)} & \secbest{40.38} & \best{48.84} & \best{65.03} & \secbest{29.08} & \best{35.22} & \best{43.61} & \secbest{21.52} & \best{26.11} & \best{39.24} & \secbest{18.50} & \best{21.61} & \best{29.14} \\
\midrule
Training Tokens & \multicolumn{3}{c}{0.44B} & \multicolumn{3}{c}{0.88B} & \multicolumn{3}{c}{2.5B} & \multicolumn{3}{c}{5.2B} \\
\bottomrule
\end{tabular}
} 
\end{table*}

\begin{table*}[t]\small
\centering
\caption{Comparison of O-LDAdam versus LDAdam~\citep{robert2024ldadam} with different subspace projection update frequencies ($T \in \{1, 200\}$) when pretraining LLaMA on C4 dataset. O-LDAdam improves LDAdam's single-step power iteration update with Ori-inspired multi-step, stochastic power iteration update. 
\textbf{In the very low-rank regimes (e.g. $r=32$ when $d_{\text{model}}=2048$), O-LDAdam with update frequency $T=200$ achieves on-par or even superior performance to LDAdam with per-step updates ($T=1$).}
}
\label{tab:pretrain-c4-ldadam}
\resizebox{\textwidth}{!}{%
\begin{tabular}{l rrr rrr rrr rrr}
\toprule
\multicolumn{1}{c}{\multirow{2}{*}{\textbf{Method}}} 
& \multicolumn{3}{c}{\textbf{60M} ($d_{\text{model}}=256$)} & \multicolumn{3}{c}{\textbf{130M} ($d_{\text{model}}=768$)} & \multicolumn{3}{c}{\textbf{350M} ($d_{\text{model}}=1024$)} & \multicolumn{3}{c}{\textbf{1B} ($d_{\text{model}}=2048$)} \\
\cmidrule(lr){2-4} \cmidrule(lr){5-7} \cmidrule(lr){8-10} \cmidrule(lr){11-13}
& $r=128$ & $r=32$ & $r=8$ & $r=256$ & $r=64$ & $r=16$ & $r=256$ & $r=64$ & $r=16$ & $r=512$ & $r=128$ & $r=32$ \\
\midrule
Adam                & & 39.03 & & & 27.60 & & & 20.68 & & & 19.92 \\
\midrule
LDAdam ($T = 1$)~\citep{robert2024ldadam}  & \best{34.57} & \best{36.19} & \secbest{39.90} & \secbest{29.32} & \secbest{30.91} & \secbest{30.02} & \best{19.97} & \best{23.52} & \secbest{23.67} & \best{17.69} & \best{16.29} & \secbest{20.27} \\
LDAdam ($T = 200$) & \secbest{37.75} & 47.90 & 69.74 & \best{27.64} & 30.95 & 41.16 & \secbest{22.38} & 25.26 & 32.03 & 161.62 & 162.92 & 45.66 \\
\midrule
\textbf{O-LDAdam (Ours)} ($T = 200$) & 41.60  & \secbest{38.52}  & \best{39.14} & 31.20 & \best{28.58} & \best{27.55} & 22.84 & \secbest{23.87} & \best{22.39} & \secbest{22.47} & \secbest{16.53} & \best{15.66} \\
\midrule
Training Tokens & \multicolumn{3}{c}{0.44B} & \multicolumn{3}{c}{0.88B} & \multicolumn{3}{c}{2.5B} & \multicolumn{3}{c}{5.2B} \\
\bottomrule
\end{tabular}
} 
\vspace{-10pt}
\end{table*}

\begin{table*}[t]\small
\centering
\caption{
Comparison with full- and low-rank optimizers on fine-tuning Roberta-Base on GLUE dataset. Average accuracy score and standard deviation is reported (lower $\downarrow$ is better) for each task, optimizer, and rank combination over three random seeds. We also report average accuracy over all tasks for each method.
We show the best scores in bold and second best ones with underlines.
\textbf{Ori consistently achieves superior performance on 10 out of 16 scenarios.}
}
\label{tab:finetune-glue} 
\resizebox{\textwidth}{!}{%
\begin{tabular}{l lrrrrrrrr r}
\toprule
\textbf{Method} & \textbf{Rank}  & \multicolumn{1}{c}{\textbf{CoLA}} & \multicolumn{1}{c}{\textbf{STS-B}} & \multicolumn{1}{c}{\textbf{MRPC}} & \multicolumn{1}{c}{\textbf{RTE}} & \multicolumn{1}{c}{\textbf{SST2}} & \multicolumn{1}{c}{\textbf{MNLI}} & \multicolumn{1}{c}{\textbf{QNLI}} & \multicolumn{1}{c}{\textbf{QQP}} & \multicolumn{1}{c}{\textbf{Avg.}} \\
\midrule
Adam & Full & 61.61{\scriptsize$\pm$0.46} & 90.88{\scriptsize$\pm$0.05} & 92.52{\scriptsize$\pm$0.12} & 78.45{\scriptsize$\pm$1.50} & 93.88{\scriptsize$\pm$0.67} & 87.41{\scriptsize$\pm$0.05} & 92.63{\scriptsize$\pm$0.37} & 89.21{\scriptsize$\pm$0.16} & 85.82\\
\midrule
LoRA~\citep{hu2022lora} & 8 & 41.85{\scriptsize$\pm$1.66} & 87.07{\scriptsize$\pm$0.25} & 87.95{\scriptsize$\pm$2.40} & 67.03{\scriptsize$\pm$2.40} & 93.12{\scriptsize$\pm$0.20} & 85.14{\scriptsize$\pm$0.10} & 90.62{\scriptsize$\pm$0.24} & 84.66{\scriptsize$\pm$0.14} & 79.68 \\
GaLore~\citep{zhao2024galore} & 8 & 58.24{\scriptsize$\pm$0.63} & 90.58{\scriptsize$\pm$0.07} & 91.06{\scriptsize$\pm$0.56} & 76.30{\scriptsize$\pm$1.46} & 93.88{\scriptsize$\pm$0.24} & 86.14{\scriptsize$\pm$0.14} & {92.41{\scriptsize$\pm$0.25}} & 86.10{\scriptsize$\pm$0.30} & 84.34 \\
OSD~\citep{liang2024memory} & 8 & \secbest{60.39{\scriptsize$\pm$0.31}} & \best{90.93{\scriptsize$\pm$0.19}} & \secbest{91.85{\scriptsize$\pm$0.50}} & \secbest{77.36{\scriptsize$\pm$1.46}} & {94.61{\scriptsize$\pm$0.30}} & {86.62{\scriptsize$\pm$0.10}} & 92.29{\scriptsize$\pm$0.16} & {87.66{\scriptsize$\pm$0.78}} & {85.16} \\
LD Adam ($T=1$) & 8 & 60.33{\scriptsize$\pm$1.07} & \secbest{90.84{\scriptsize$\pm$0.05}} & 91.49{\scriptsize$\pm$0.16} & 76.65{\scriptsize$\pm$0.74} & \secbest{94.64{\scriptsize$\pm$0.27}} & 86.98{\scriptsize$\pm$0.27} & \best{92.70{\scriptsize$\pm$0.31}} & \best{89.90{\scriptsize$\pm$1.33}} & \secbest{85.96} \\
LD Adam ($T=200$) & 8 & 59.07{\scriptsize$\pm$0.87} & 90.54{\scriptsize$\pm$0.07} & 91.05{\scriptsize$\pm$0.47} & 76.29{\scriptsize$\pm$1.33} & 94.61{\scriptsize$\pm$0.34} & \secbest{87.13{\scriptsize$\pm$0.07}} & 92.46{\scriptsize$\pm$0.14} & 88.54{\scriptsize$\pm$0.16} & 84.96  \\
\midrule
\textbf{Ori (Ours)}& 8  & \best{63.58{\scriptsize$\pm$0.93}} & {90.81{\scriptsize$\pm$0.21}} & \best{92.93{\scriptsize$\pm$0.23}} & \best{79.42{\scriptsize$\pm$1.08}} & \best{94.92{\scriptsize$\pm$0.54}} & \best{87.52{\scriptsize$\pm$0.05}} & \secbest{92.55{\scriptsize$\pm$0.07}} & \secbest{89.27{\scriptsize$\pm$0.04}} & \best{86.30} \\
\textbf{O-LDAdam (Ours)} & 8 & 58.38{\scriptsize$\pm$0.67} & 90.43{\scriptsize$\pm$0.09} & 91.18{\scriptsize$\pm$0.12} & 73.65{\scriptsize$\pm$2.06} & 94.57{\scriptsize$\pm$0.11} & \secbest{87.13{\scriptsize$\pm$0.13}} & 92.18{\scriptsize$\pm$0.25} & 88.32{\scriptsize$\pm$0.13} & 84.48 \\
\toprule
LoRA~\citep{hu2022lora} & 4 & 41.93{\scriptsize$\pm$1.20} & 86.77{\scriptsize$\pm$0.07} & 87.52{\scriptsize$\pm$0.98} & 66.30{\scriptsize$\pm$1.46} & 93.00{\scriptsize$\pm$0.30} & 84.62{\scriptsize$\pm$0.16} & 90.31{\scriptsize$\pm$0.12} & 84.37{\scriptsize$\pm$0.06} & 79.35 \\
GaLore~\citep{zhao2024galore} & 4 & {60.47{\scriptsize$\pm$2.16}} & {90.67{\scriptsize$\pm$0.23}} & {91.72{\scriptsize$\pm$0.48}} & 77.01{\scriptsize$\pm$2.33} & 94.19{\scriptsize$\pm$0.24} & {87.24{\scriptsize$\pm$0.15}} & {92.24{\scriptsize$\pm$0.14}} & {87.96{\scriptsize$\pm$0.10}} & 85.19\\
OSD~\citep{liang2024memory} & 4 & 59.93{\scriptsize$\pm$1.04} & 90.55{\scriptsize$\pm$0.16} & 91.67{\scriptsize$\pm$0.38} & \secbest{77.62{\scriptsize$\pm$2.01}} & \secbest{94.30{\scriptsize$\pm$0.35}} & 86.00{\scriptsize$\pm$0.19} & 92.17{\scriptsize$\pm$0.16} & 87.34{\scriptsize$\pm$0.47} & 84.95\\
LD Adam ($T=1$) & 4 & \best{61.24{\scriptsize$\pm$1.56}} & \secbest{90.78{\scriptsize$\pm$0.15}} & \secbest{91.81{\scriptsize$\pm$0.21}} & 77.01{\scriptsize$\pm$1.36}& \best{94.53{\scriptsize$\pm$0.46}} & \secbest{87.25{\scriptsize$\pm$0.11}} & \secbest{92.39{\scriptsize$\pm$0.14}} & \best{89.92{\scriptsize$\pm$1.35}} & \best{85.61} \\
LD Adam ($T=200$) & 4 & 58.48{\scriptsize$\pm$0.33} & 90.27{\scriptsize$\pm$0.10} & 91.51{\scriptsize$\pm$0.08} & 72.68{\scriptsize$\pm$0.45} & \secbest{94.45{\scriptsize$\pm$0.28}} & 87.07{\scriptsize$\pm$0.02} & 92.38{\scriptsize$\pm$0.18} & 88.45{\scriptsize$\pm$0.16} & 84.41  \\
\midrule
\textbf{Ori (Ours)} & 4 & \secbest{60.63{\scriptsize$\pm$0.40}} & \best{90.92{\scriptsize$\pm$0.11}} & \best{91.91{\scriptsize$\pm$0.41}} & \best{77.98{\scriptsize$\pm$1.30}} & {94.42{\scriptsize$\pm$0.24}} & \best{87.30{\scriptsize$\pm$0.13}} & \best{92.62{\scriptsize$\pm$0.23}} & \secbest{89.06{\scriptsize$\pm$0.15}} & \best{85.61}  \\
\textbf{O-LDAdam (Ours)} & 4 & 58.09{\scriptsize$\pm$0.75} & 90.38{\scriptsize$\pm$0.05} & 90.34{\scriptsize$\pm$0.67} & 71.12{\scriptsize$\pm$1.47} & 94.23{\scriptsize$\pm$0.20} & 87.06{\scriptsize$\pm$0.04} & 92.18{\scriptsize$\pm$0.06} & 88.38{\scriptsize$\pm$0.07} & 83.97 \\
\bottomrule
\end{tabular}
}
\vspace{-5pt}
\end{table*}

\subsection{Main results: LLaMA pre-training}
\label{sec:exp-pretrain}

We evaluate Ori's pre-training performance against the baselines on the C4 dataset~\citep{raffel2020exploring} using LLaMA-architecture models of four sizes (60M, 130M, 350M, and 1B), following similar experimental settings as prior work~\citep{lialin2023relora,zhao2024galore}. Departing from prior work that often tests a single rank, we systematically evaluate all low-rank methods across three distinct ranks for each model (e.g., $r \in \{512, 128, 32\}$ for the 1B model), enabling a rigorous assessment of optimizer performance under varying degrees of memory constraints. To ensure a fair comparison, model hyperparameters were held constant for each model size, while each optimizer's learning rate was tuned from $\{0.01, 0.005, 0.001, 0.0005, 0.0001\}$. We report validation perplexity after a fixed token budget, 
with full details provided in \Appref{app:exp-details}.

As shown in \Cref{tab:pretrain-c4}, Ori's performance is consistently strong, particularly in memory-constrained scenarios. 
At higher, less restrictive ranks (e.g., $r=128$ for the 60M model, $r=256$ for the 130M model), Ori performs on par with GaLore, establishing it as a highly competitive method in common low-rank settings. 
However, Ori's distinct advantage emerges when the rank $r$ is severely constrained, where it demonstrates the most robust performance as its accuracy degrades the least under memory pressure. This trend, where the cost of a stale subspace projection becomes most apparent, is consistent across all model scales. For instance, when training the 130M model at an aggressively low rank of just 16, Ori achieves a validation perplexity of 43.61, substantially outperforming GaLore (60.29) and OSD (54.05). This result empirically validates our central hypothesis: actively tracking the gradient subspace drift is most critical when the optimizer has the fewest dimensions to work with. By excelling in these challenging low-rank regimes, Ori consistently narrows the performance gap to full-rank Adam, pushing the performance boundaries of memory-efficient optimization.

\vspace{-5pt}
\paragraph{Generalizing to Error-Feedback Optimizers (O-LDAdam).} 
To test the generality of our subspace tracking approach, we benchmarked against LDAdam, a stronger baseline that utilizes error feedback and momentum buffering. While LDAdam achieves impressive performance when updating the projector at every step ($T=1$), this comes at a high computational cost. As shown in Table 2, simply increasing the update interval to $T=200$ (matching GaLore’s efficiency) causes LDAdam's performance to collapse (e.g., perplexity spikes from 20.27 to 45.66 on LLaMA-1B at $r=32$). However, by replacing LDAdam’s simple power iteration with our stochastic power iteration (O-LDAdam), we effectively bridge this gap. O-LDAdam ($T=200$) not only recovers the performance of the expensive LDAdam ($T=1$) baseline but frequently surpasses it in extremely low-rank regimes. Notably, on the 1B model at rank 32, O-LDAdam achieves a perplexity of 15.66, significantly outperforming the vanilla LDAdam ($T=1$) score of 20.27, all while maintaining the computationally efficient $T=200$ schedule. This confirms that the subspace tracking principles of Ori are orthogonal to error-feedback mechanisms and can be composed to create state-of-the-art low-rank optimizers.

\subsection{Finetuning}
\label{sec:exp-finetune}

To evaluate Ori's effectiveness on adaptation tasks, we fine-tuned a RoBERTa-Base model on the GLUE benchmark~\citep{wang2018glue}. For each task, we report the mean and standard deviation of task-specific metrics over three runs with different random seeds. Further details on evaluation metrics are provided in \Appref{app:exp-details}.

The fine-tuning results in Table 3 underscore Ori's prominent performance. Across both rank settings, Ori achieves the highest average score among all low-rank methods. At rank-8, Ori's average score of 86.30 not only surpasses GaLore (84.34) and LoRA (79.68) but also exceeds the full-rank Adam baseline (85.82). This strong performance continues at rank-4, where Ori (85.61) again leads all low-rank optimizers except for LDAdam with $T=1$. 
In particular,
Ori achieves the top score in 10 out of the 16 total task/rank combinations. Its particularly strong advantage on challenging tasks like CoLA (63.58 vs. GaLore's 58.24 at r=8) further validates its robustness and demonstrates that its drift-aware optimization is highly effective for both pre-training and adaptation tasks.

\subsection{Ablation Study}
\label{sec:exp-ablation}

\begin{figure}[t]
    \centering
    \begin{subfigure}[b]{0.48\textwidth}
        \centering
        \includegraphics[width=\textwidth]{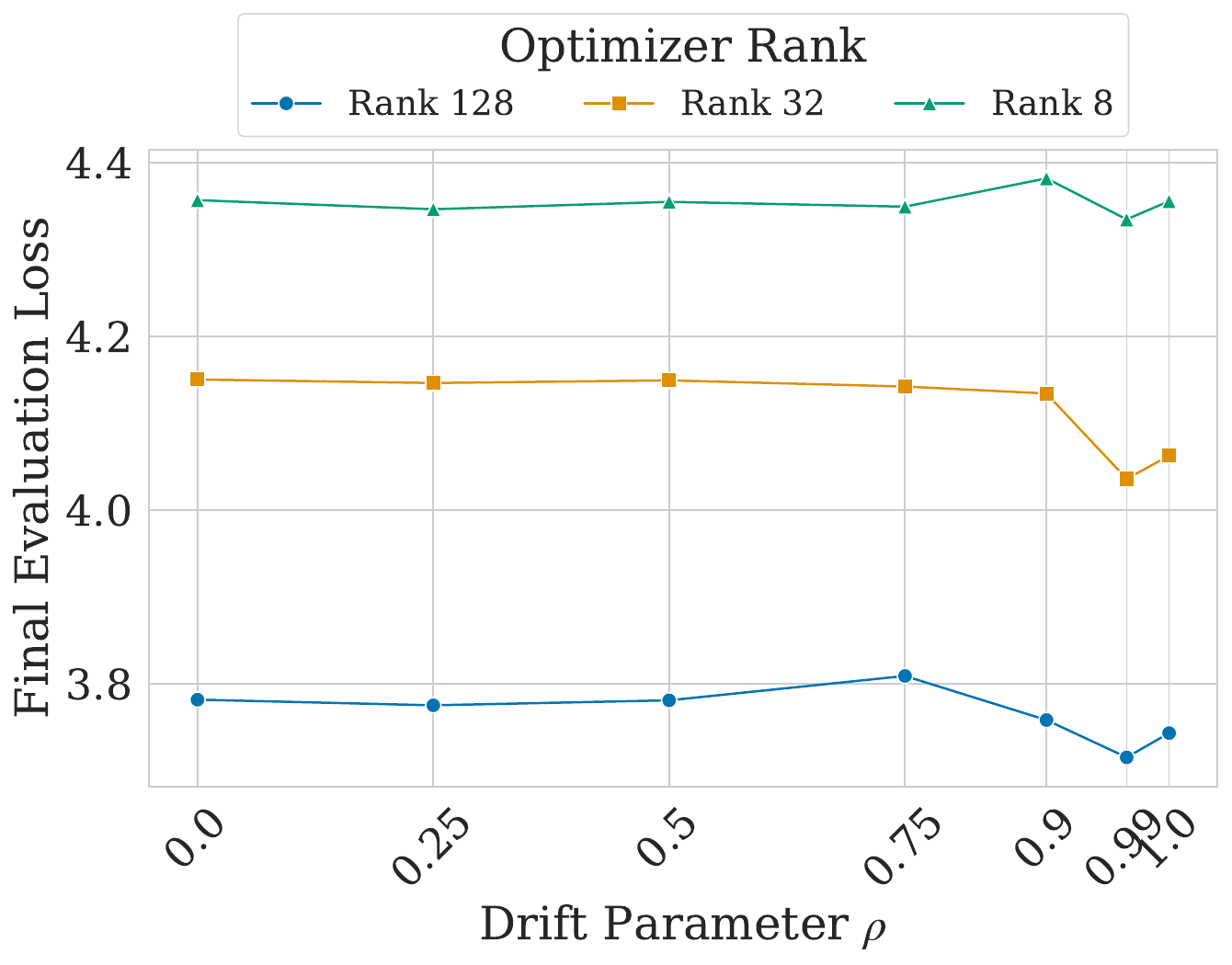}
        \caption{Impact of Drift Coefficient ($\rho$) on final evaluation loss. Lower $\downarrow$ is better.}
        \label{fig:ablation_rho}
    \end{subfigure}
    \hfill 
    \begin{subfigure}[b]{0.48\textwidth}
        \centering
        \includegraphics[width=\textwidth]{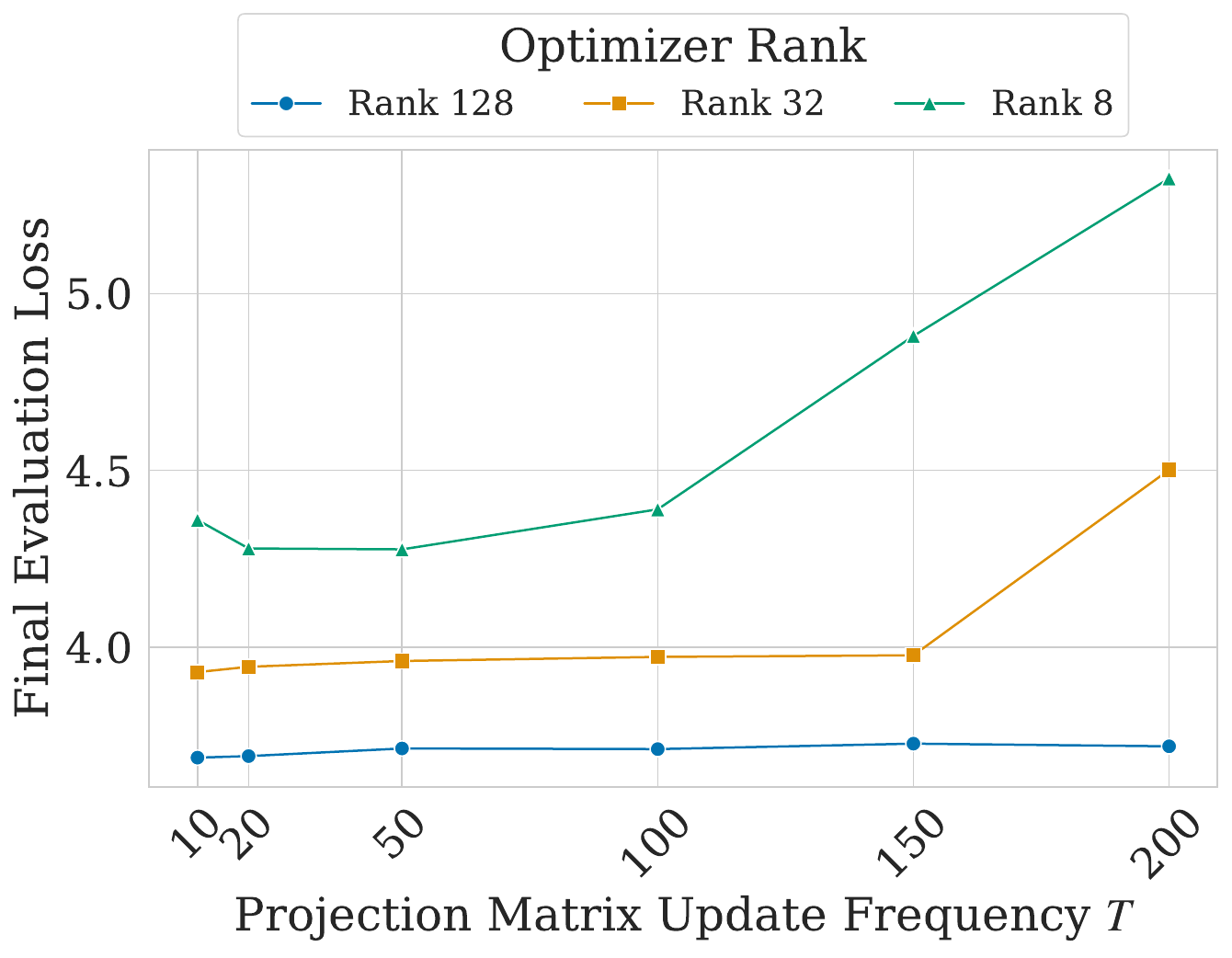}
        \caption{Impact of Update Frequency ($T_{\text{period}}$) on final evaluation loss. Lower $\downarrow$ is better.}
        \label{fig:ablation_freq}
    \end{subfigure}

    \begin{subfigure}{0.6\textwidth}\small
        \centering
        \caption{Oja++ vs. Standard Oja (Perplexity).}
        \label{tab:ablation_oja}
        \resizebox{0.9\textwidth}{!}{
        \begin{tabular}{lccc}
            \toprule
            \textbf{Algorithm} & \textbf{Rank (r=128)} & \textbf{Rank (r=32)} & \textbf{Rank (r=8)} \\
            \midrule
            Standard Oja & 41.10 & 52.86 & 96.13 \\
            \textbf{Oja++ (Ours)} & \textbf{40.38} & \textbf{48.84} & \textbf{65.03} \\
            \bottomrule
        \end{tabular}
        }
    \end{subfigure}
    \caption{
        Ablation Studies on the 60M LLaMA model pre-trained on C4 dataset for 0.44B tokens.
        (\subref{fig:ablation_rho}) \textbf{Varying the drift coefficient $\rho$ shows peak performance around $\rho=0.99$}, validating the importance of transferring optimizer statistics.
        (\subref{fig:ablation_freq}) \textbf{Performance degrades as the update frequency $T_{\text{period}}$ increases}, demonstrating the cost of a stale projector.
        (\subref{tab:ablation_oja}) \textbf{Oja++ consistently outperforms the standard Oja algorithm across all ranks.}
    }
    \label{fig:ablations}
    \vspace{-10pt}
\end{figure}

To isolate the contributions of Ori's core design components, we conduct a series of ablation studies, performed via the 60M LLaMA model pre-trained on the C4 dataset (\Cref{sec:exp-pretrain}). Details about the ablation experiment settings can be found in \Appref{app:ablation-details}.

\vspace{-10pt}
\paragraph{Impact of the Drift Coefficient ($\rho$).} We first analyze the effect of the drift coefficient $\rho$, which controls how much of the previous step's projector is carried over to initialize the current step's projector. As shown in \Cref{fig:ablation_rho}, we observe a clear and consistent trend across all ranks: a value of $\rho=0.99$ yields the best performance. This has two key implications. First, the significant performance drop at $\rho=0.0$ (a fully random restart) corroborates our central claim that transferring information about the previous subspace is crucial for effective optimization. Second, the slight performance dip at $\rho=1.0$ (a direct carryover with no noise) suggests that injecting a small amount of randomness is a beneficial regularization strategy. Overall, we choose $\rho=0.99$ for Ori in all pre-training and fine-tuning.

\vspace{-10pt}
\paragraph{Impact of Projector Update Frequency ($T$)} We test our hypothesis that frequent, lightweight updates are superior to infrequent, heavy ones by varying the update period from 10 to 200 (which matches GaLore's recommended schedule~\citep{zhao2024galore}). As shown in \Cref{fig:ablation_freq}, more frequent updates consistently lead to lower loss (better performance). This phenomenon is especially prominent in the extremely low-rank setting ($r=8$), where the cost of using a stale projector is most severe. 
This finding contrasts sharply with the ablation results from GaLore~\citep{zhao2024galore}, in which the authors reported no benefit from more frequent GaLore SVD updates. We hypothesize this is a unique advantage of Ori: because GaLore's periodic SVD acts as a ``cold restart'' that does not carry over subspace statistics, frequent updates may not be as beneficial. In contrast, since Ori is explicitly designed to track subspace drift, it can effectively leverage more frequent updates to maintain a higher-fidelity representation of the gradient landscape.
For our main experiments, we set $T = 20$, as it strikes an effective balance between model performance and computational overhead.

\vspace{-10pt}
\paragraph{Oja++ vs. Standard Oja.} To empirically justify choosing Oja++, we compare the performance of Oja++ against the standard Oja algorithm as the subroutine for updating the gradient projector. The results in \Cref{tab:ablation_oja} are unequivocal: Oja++ consistently outperforms standard Oja across all tested ranks. This aligns with prior theoretical work showing Oja++ has a provably better convergence rate in static streaming $k$-PCA settings~\citep{allen2017first}. Our results provide new empirical evidence that this advantage extends to the dynamic, non-stationary environment of LLM training, where Oja++'s superior tracking ability translates directly to better model performance.

\vspace{-10pt}
\paragraph{Memory and Computation Time.}
Finally, aside from the theoretical runtime complexity analysis in \Appref{app:complexity}, we provide an empirical memory and computation time comparison in \Appref{sec:exp-memory-time}.

\section{Conclusion}
We introduced Ori, a drift-aware low-rank optimizer that continuously tracks the gradient subspace with streaming $k$-PCA and preserves optimizer statistics across projector updates, mitigating staleness while keeping memory and compute low. Across theory (dynamic-regret under drift) and experiments, Ori consistently narrows the gap to full-rank training in the extremely low-rank regime while remaining faster and more memory-efficient than prior gradient-projection methods. 
We discuss limitations of our work and directions for future research, including behavior under heterogeneous training data and explicit optimizer state rotation across basis changes, in~\Appref{sec:limitations}.

\bibliography{references}
\bibliographystyle{plainnat}

\newpage
\appendix

\section{Proofs}  \label{app:proofs}

A central tool is the identity relating projectors and principal angles.

\begin{definition}[Principal angles]
For orthonormal $Q_1,Q_2\in\R^{d\times k}$, let $\Theta(Q_1,Q_2)=\diag(\theta_1,\ldots,\theta_k)$ with $\cos\theta_i$ the singular values of $Q_1^\top Q_2$.
\end{definition}

\begin{lemma}\label{lem:dist-identity} (Distance Identity.)
    Let $P_1 = Q_1 Q_1^\top$ and $P_2 = Q_2 Q_2^\top$ be two rank-$k$ orthogonal projection matrices, where $Q_1, Q_2 \in \R^{d\times k}$ are matrices with orthonormal columns. Then, we have the following:
    \begin{align*}
        \norm{P_1 - P_2}_F^2 = 2\norm{\sin{\Theta(Q_1, Q_2)}}_F^2.
    \end{align*}
\end{lemma}

\begin{proof}
Provided in~\Cref{app:proof-dist-identity}.
\end{proof}

\Cref{lem:dist-identity} converts the drift assumption into an angle bound:
\begin{equation}
  \norm{\sin\Theta(V_s,V_{s-1})}_F^2 \;\le\; \lambda^2/2.
  \label{eq:angle-drift}
\end{equation}

\paragraph{Per-step potential drop.}
Define the misalignment potential $d_\zeta^2:=\norm{\sin\Theta(Q^{(\zeta)},V_s)}_F^2$. Analyses of noisy Oja yield a one-step inequality that links potential drop to the instantaneous regret:
\begin{equation}
  \E\!\left[d_{\zeta+1}^2\right]-d_\zeta^2
  \;\le\;
  -\,2\eta\,\inp{P_s-P^{(\zeta)}}{\hat g_\zeta} \;+\; \gO(\eta^2),
  \qquad P^{(\zeta)}:=Q^{(\zeta)}(Q^{(\zeta)})^\top,
  \label{eq:oja-descent-sketch}
\end{equation}
where the $\gO(\eta^2)$ term represents the second-order error from the Taylor expansion of the potential function. 
As detailed in~\Cref{app:per-step-drop-detailed}, the magnitude of this geometrically-induced error is bounded by a constant proportional to $\eta^2\norm{\hat{g}_\zeta}_F^2$. This quantity is, in turn, controlled by our assumption of bounded input data ($\norm{x_\zeta}_2\leq 1$), which ensures the stability of the analysis. 

Decomposing $\hat g_\zeta = x_\zeta x_\zeta^\top + E_\zeta$ with the error term $E_\zeta:=x_\zeta y_\zeta^\top+y_\zeta x_\zeta^\top+y_\zeta y_\zeta^\top$, we get
\[
  \inp{P_s-P^{(\zeta)}}{\hat g_\zeta}
  \;=\;
  \inp{P_s-P^{(\zeta)}}{x_\zeta x_\zeta^\top}
  \;+\;
  \inp{P_s-P^{(\zeta)}}{E_\zeta}.
\]
Hölder's inequality and $\|P_s\|_*=\|P^{(\zeta)}\|_*=k$ imply
$|\inp{P_s-P^{(\zeta)}}{E_\zeta}|\le \|P_s-P^{(\zeta)}\|_*\,\|E_\zeta\|_2 \le 2k\,\|E_\zeta\|_2$.
Rearranging \Cref{eq:oja-descent-sketch} and using $\gO(\eta^2)\le 4kd\cdot \eta^2$ gives
\begin{equation}
  \inp{P_s-P^{(\zeta)}}{x_\zeta x_\zeta^\top}
  \;\le\;
  \frac{d_\zeta^2 - \E[d_{\zeta+1}^2]}{2\eta} + 2kd\eta + 2k\,\|E_\zeta\|_2.
  \label{eq:per-iter-regret}
\end{equation}
Detailed proof for \Cref{eq:per-iter-regret} is shown at \Cref{app:per-step-drop-detailed}. Summing \Cref{eq:per-iter-regret} over $\zeta\in\{(s-1)T+1,\dots,sT\}$ yields the per-step bound
\begin{equation}
  R_s
  \;\le\;
  \frac{d_{(s-1)T+1}^2 - \E\!\left[d_{sT+1}^2\right]}{2\eta} \;+\; 2kd\eta T \;+\; 2k\,E_s,
  \qquad
  E_s:=\sum_{\zeta=(s-1)T+1}^{sT}\!\!\|E_\zeta\|_2.
  \label{eq:ojs-step}
\end{equation}

\paragraph{Linking steps via drift in angle space.}
By \Cref{lem:dist-identity}, we have that
\[
  d_s(V_s)^2 - d_s(V_{s-1})^2
  \;=\;
  \tfrac{1}{2}\!\left(\norm{P_Q-P_s}_F^2 - \norm{P_Q-P_{s-1}}_F^2\right)
  \;\le\; \sqrt{k}\,\lambda,
\]
where $P_Q:=Q^{((s-1)T+1)}(Q^{((s-1)T+1)})^\top$. With any orthonormal initialization $Q^{(1)}$, we have $d_1(V_1)^2\le k$.

\subsection{Proof of Lemma ~\ref{lem:dist-identity}: Distance Identity}\label{app:proof-dist-identity}
\begin{proof}
    \begin{align*}
        \norm{P_1 - P_2}_F^2 = Tr\left((P_1 - P_2)^\top(P_1 - P_2)\right) = Tr(P_1^2 - P_1P_2 - P_2P_1 + P_2^2)
    \end{align*}
    Since $P_1$ and $P_2$ are orthogonal projectors, they are symmetric $P_i^\top = P_i$ and idempotent ($P_i^2 = P_i$). The trace of a rank-$k$ projector is its rank, so $Tr(P_1) = Tr(P_2) = k$. The expression simplifies to:
    \begin{align*}
        ||P_1 - P_2||_F^2 = \text{Tr}(P_1) + \text{Tr}(P_2) - 2\text{Tr}(P_1 P_2) = 2k - 2\text{Tr}(P_1 P_2).
    \end{align*}
    The trace term can be rewritten as:
    \begin{align*}
        \text{Tr}(P_1 P_2) = \text{Tr}(Q_1 Q_1^\top Q_2 Q_2^\top) = \text{Tr}(Q_2^\top Q_1 Q_1^\top Q_2) = ||Q_1^\top Q_2||_F^2.
    \end{align*}
    The singular values of the matrix $Q_1^\top Q_2$ are the cosines of the principal angles, $\{\cos\theta_i\}_{i=1}^k$. The squared Frobenius norm is the sum of the squares of the singular values:
    \begin{align*}
        \norm{Q_1^\top Q_2}_F^2 = \sum_{i=1}^k \cos^2 \theta_i.
    \end{align*}
    Substituting this back gives:
    \begin{align*}
        \norm{P_1 - P_2}_F^2 = 2k - 2\sum_{i = 1}^k\cos^2\theta_i = 2\sum_{i=1}^k(1 - \cos^2\theta_i) = 2\sum_{i=1}^k\sin^2\theta_i = 2\norm{\sin{\Theta(Q_1, Q_2)}}_F^2.
    \end{align*}
\end{proof}

\subsection{Detailed proof of~\eqref{eq:oja-descent-sketch} and~\eqref{eq:per-iter-regret}(one-step potential drop for Oja)}\label{app:per-step-drop-detailed}

Recall $P^{(\zeta)} = Q^{(\zeta)}(Q^{(\zeta)})^\top$, $P_s = V_sV_s^\top$, and define the misalignment potential
\[
d_\zeta^2 \;:=\; \norm{\sin\Theta\!\big(Q^{(\zeta)},V_s\big)}_F^2 \;=\; \tfrac{1}{2}\norm{P^{(\zeta)}-P_s}_F^2.
\]
Denote $P:=P^{(\zeta)}$ and $Z:=\hat g_\zeta$ for brevity, where $Z$ is symmetric. Oja's update reads
\begin{equation}
Q^{(\zeta+1)} \;=\; \mathrm{QR}\!\Big((I+\eta Z)\,Q^{(\zeta)}\Big), 
\qquad P^{(\zeta+1)} \;=\; Q^{(\zeta+1)}(Q^{(\zeta+1)})^\top.
\label{eq:app-oja}
\end{equation}
Here is the proof of~\eqref{eq:oja-descent-sketch}
\[
\E\!\left[d_{\zeta+1}^2\right] - d_\zeta^2
\;\le\;
-\,2\eta\,\inp{P_s - P}{Z} \;+\; \gO(\eta^2),
\]
\begin{proof}
    Consider the pre-retraction iterate $\widetilde Q=(I+\eta Z)Q^{(\zeta)}$ and its (non-idempotent) projector $\widetilde P:=\widetilde Q\widetilde Q^\top$:
\begin{align}
\widetilde P
&= (I+\eta Z)\,P^2\,(I+\eta Z)
= (I+\eta Z)\,P\,(I+\eta Z)
= P \;+\; \eta(ZP+PZ) \;+\; \eta^2 Z P Z.
\label{eq:app-preQR}
\end{align}
Let $\mathsf{T}_P:=\{\Delta=\Delta^\top:\ P\Delta P=0,\ (I-P)\Delta(I-P)=0\}$ denote the tangent space of the Grassmannian at $P$.
Projecting the first-order term onto $\mathsf{T}_P$ gives the canonical tangent increment
\begin{equation}
\Delta_P \;:=\; \Pi_{\mathsf{T}_P}\big(ZP+PZ\big) \;=\; (I-P)\,Z\,P \;+\; P\,Z\,(I-P).
\label{eq:app-tan}
\end{equation}
Standard properties of retractions (here, $\mathrm{QR}$) imply that the reorthonormalized projector satisfies
\begin{equation}
P^{(\zeta+1)} \;=\; P \;+\; \eta\,\Delta_P \;+\; \gO(\eta^2),
\label{eq:app-retr}
\end{equation}
i.e., retraction preserves the first-order tangent component and pushes all curvature effects into $\gO(\eta^2)$; which is a similar "one-step" analysis strategy of Appendix i.B in \citet{allen2017first}.

Define $\phi(P)=\tfrac{1}{2}\norm{P-P_s}_F^2$, whose Euclidean gradient is $\nabla\phi(P)=P-P_s$. By expanding, we have
\begin{align}
d_{\zeta+1}^2 - d_\zeta^2
&= \phi\!\big(P^{(\zeta+1)}\big)-\phi(P)
= \inp{P-P_s}{P^{(\zeta+1)}-P} \;+\; \frac{1}{2}\norm{P^{(\zeta + 1)} - P}_F^2
\nonumber\\
&= \eta\,\inp{P-P_s}{\Delta_P} \;+\; \frac{\eta^2}{2}\norm{\Delta_P}_F^2.
\label{eq:app-diff}
\end{align}
\\Write
\[
\Delta_P \;=\; (I-P)ZP + PZ(I-P) \;=\; ZP + PZ - 2PZP.
\]
A direct computation (e.g., in a block basis that diagonalizes $P$) yields the identity
\begin{equation}
\inp{P-P_s}{\Delta_P}
\;=\;
-\,2\,\inp{\Pi_{\mathsf{T}_P}(P_s - P)}{Z},
\qquad
\Pi_{\mathsf{T}_P}(P_s-P) \;=\; (I-P)P_sP + P P_s(I-P).
\label{eq:app-key-id}
\end{equation}
\emph{Proof of \eqref{eq:app-key-id} (sketch).}
Work in the orthonormal basis $[Q\ \ Q_\perp]$ with $P=\diag(I_k,0)$
and decompose
\(
Z=\begin{psmallmatrix}A & B^\top \\ B & C\end{psmallmatrix},
\;
P_s=\begin{psmallmatrix}X & Y^\top \\ Y & Z_0\end{psmallmatrix}.
\)
Then
\(
\Delta_P=\begin{psmallmatrix}0 & B^\top\\ B & 0\end{psmallmatrix}
\)
and
\(
\Pi_{\mathsf{T}_P}(P_s-P)=\begin{psmallmatrix}0 & Y^\top\\ Y & 0\end{psmallmatrix}.
\)
Hence
\(
\inp{P-P_s}{\Delta_P}
= -2\,\Tr(Y^\top B)
= -\,2\,\inp{\Pi_{\mathsf{T}_P}(P_s-P)}{Z}.
\)
\hfill$\square$

Combining \eqref{eq:app-diff} and \eqref{eq:app-key-id} we obtain
\begin{equation}
d_{\zeta+1}^2 - d_\zeta^2
\;=\;
-\,2\eta\,\inp{\Pi_{\mathsf{T}_P}(P_s - P)}{Z} \;+\; \gO(\eta^2).
\label{eq:app-tangent-form}
\end{equation}

We now relate $\inp{\Pi_{\mathsf{T}_P}(P_s-P)}{Z}$ to the full inner product $\inp{P_s-P}{Z}$ that appears in the instantaneous regret.
Observe that
\begin{align}
\inp{P_s-P}{Z}
&= \inp{\Pi_{\mathsf{T}_P}(P_s-P)}{Z}
\;+\; \underbrace{\inp{(I-P)(P_s-P)(I-P)}{Z}}_{\ge 0 \text{ if } Z\succeq 0}
\;+\; \underbrace{\inp{P\,(P_s-P)\,P}{Z}}_{\le 0 \text{ if } Z\succeq 0}
\nonumber\\
&= \inp{\Pi_{\mathsf{T}_P}(P_s-P)}{Z}
\;+\; \inp{(I-P)P_s(I-P)}{Z}
\;-\; \inp{P - P P_s P}{Z}.
\label{eq:app-split}
\end{align}
Since $0\preceq P_s\preceq I$ and $Z\succeq 0$, both $(I-P)P_s(I-P)$ and $(P - P P_s P)$ are positive semidefinite, and hence the last line of \eqref{eq:app-split} shows
\begin{equation}
\inp{\Pi_{\mathsf{T}_P}(P_s-P)}{Z}
\;\le\;
\inp{P_s-P}{Z}
\qquad \text{whenever } Z\succeq 0.
\label{eq:app-monotone}
\end{equation}
Plugging \eqref{eq:app-monotone} into \eqref{eq:app-tangent-form} yields
\begin{equation}
d_{\zeta+1}^2 - d_\zeta^2
\;\le\;
-\,2\eta\,\inp{P_s - P}{Z} \;+\; \gO(\eta^2),
\qquad (Z\succeq 0).
\label{eq:app-main-drop}
\end{equation}

To be explicit about the control of the $\gO(\eta^2)$ higher-order terms, we trace their origin. The term $P^{(\zeta +1)} -P$ in the Taylor expansion represents the full update step on the manifold. As established in \eqref{eq:app-preQR} and the subsequent discussion on retractions, this step can be decomposed into a first-order tangent component $\eta \Delta_P$, and higher-order components. These higher-order components, which we collect into the $\gO(\eta^2)$ term, are a direct consequence of two factors: (i) the explicit algebraic term $\eta^2ZPZ$ that arises from the pre-retraction update, and (ii) the intrinsic curvature of the manifold, which means the path defined by the QR retraction deviates from the straight-line path of the tangent direction.

Recalling the algebraic second-order term, we have
\begin{align*}
    \tilde{P} = P + \eta(ZP+PZ) + \eta^{2}ZPZ,\quad \norm{\eta^2 ZPZ}_F \le \eta^2 \norm{Z}_F \norm{P}_F \norm{Z}_F = \eta^2 \norm{P}_F \norm{Z}_F^2.
\end{align*}
Since $P$ is a rank-k orthogonal projector, its Frobenius norm is a constant: $\norm{P}_F = \sqrt{Tr(P^\top P)} = \sqrt{Tr(P)} = \sqrt{k}$. Therefore:
\begin{align*}
    \norm{\eta^2ZPZ}_F \leq \sqrt{k}\cdot \eta^2\norm{Z}_F^2.
\end{align*}

For geometric second-order term, we need to relate $\|\Delta_P\|_F^2$ to $\|Z\|_F^2$. From~\eqref{eq:app-tan}, we have $\Delta_P = (I - P) Z P + P Z (I - P)$.
Using the triangle inequality and the sub-multiplicativity of the Frobenius norm, we get:
\begin{align*}
\norm{\Delta_P}_F \leq \|(I - P) Z P\|_F + \|P Z (I - P)\|_F \leq 2 \|P\|_F \|I - P\|_F \|Z\|_F.
\end{align*}

Since $P$ is a rank-$k$ projector and $I - P$ is a rank-$(d - k)$ projector, their norms are constants 
\begin{align*}
\|P\|_F = \sqrt{k} \quad \text{and} \quad \|I - P\|_F = \sqrt{d - k},
\end{align*}
respectively. This means $\|\Delta_P\|_F$ is bounded by a constant times $\|Z\|_F$. Consequently, the second-order geometric error $\eta^2\|\Delta_P\|_F^2$ is bounded by $(4kd - 4k^2)\eta^2\|Z\|_F^2$. In total, we have $\gO(\eta^2)\leq 4kd\cdot \eta\norm{Z}_F^2$.
The norm of these higher-order terms is therefore fundamentally governed by the norm of the update matrix $Z = \hat{g}_\zeta$. 
Under $\norm{x_\zeta}_2\le 1$ (so $\norm{Z}_F=\norm{x_\zeta x_\zeta^\top}_F=\norm{x_\zeta}_2^2\le 1$) we have $\gO(\eta^2)\le 4kd\cdot \eta^2$. 
Finally, taking expectation over the randomness at time $\zeta$ (data and any internal noise) yields
\[
\E\!\left[d_{\zeta+1}^2\right] - d_\zeta^2
\;\le\;
-\,2\eta\,\inp{P_s - P}{\hat g_\zeta} \;+\; 4kd\cdot \eta^2,
\]
which is Eq.~\eqref{eq:oja-descent-sketch}.
\end{proof}

If $\hat g_\zeta = x_\zeta x_\zeta^\top + E_\zeta$ with an arbitrary symmetric perturbation $E_\zeta$, the argument above applies verbatim to the PSD part $x_\zeta x_\zeta^\top$; the additional term satisfies
\[
\big|\inp{P_s-P}{E_\zeta}\big|
\;\le\;
\norm{P_s-P}_*\,\norm{E_\zeta}_2
\;\le\; 2k\,\norm{E_\zeta}_2,
\]
since $\norm{P_s}_*=\norm{P}_*=k$. This gives the noise contribution used in the main text.
\qed

\subsection{Proof of Theorem \ref{thm:oja}}\label{app:theo2}
\begin{proof}
Fix a step $s\in[S]$ and write $P^{(\zeta)} = Q^{(\zeta)}(Q^{(\zeta)})^\top$ and $P_s=V_sV_s^\top$.
Define the misalignment potential
\(
  d_\zeta^2 := \norm{\sin\Theta(Q^{(\zeta)},V_s)}_F^2
  = \tfrac12 \norm{P^{(\zeta)}-P_s}_F^2
\)
(Lemma~\ref{lem:dist-identity}).
By the one-step potential drop for Oja (proved in \Appref{app:per-step-drop-detailed}), for the noisy update
\(
  Q^{(\zeta+1)}=QR\!\big((I+\eta\,\hat g_\zeta)Q^{(\zeta)}\big)
\)
with $\hat g_\zeta=\hat x_\zeta\hat x_\zeta^\top$ and $\norm{x_\zeta}_2\le 1$, we have
\begin{equation}\label{eq:step-drop-proof}
  \E\!\left[d_{\zeta+1}^2\right]-d_\zeta^2
  \;\le\;
  -\,2\eta\,\inp{P_s-P^{(\zeta)}}{\hat g_\zeta}
  \;+\; 4kd\eta^2.
\end{equation}
Decompose the noisy rank-one matrix as
\(
  \hat g_\zeta = x_\zeta x_\zeta^\top + E_\zeta,
\)
where $E_\zeta:=x_\zeta y_\zeta^\top+y_\zeta x_\zeta^\top+y_\zeta y_\zeta^\top$.
Using H\"older and $\|P_s\|_*=\|P^{(\zeta)}\|_*=k$,
\[
  \big|\inp{P_s-P^{(\zeta)}}{E_\zeta}\big|
  \;\le\; \|P_s-P^{(\zeta)}\|_*\,\|E_\zeta\|_2
  \;\le\; 2k\,\|E_\zeta\|_2.
\]
Rearranging \eqref{eq:step-drop-proof} gives the per-iteration regret bound
\begin{equation}\label{eq:per-iter-proof}
  \inp{P_s-P^{(\zeta)}}{x_\zeta x_\zeta^\top}
  \;\le\;
  \frac{d_\zeta^2 - \E[d_{\zeta+1}^2]}{2\eta}
  \;+\; 2kd\eta
  \;+\; 2k\,\|E_\zeta\|_2 .
\end{equation}

\paragraph{Within a step.}
Summing \eqref{eq:per-iter-proof} over $\zeta=(s-1)T+1,\dots,sT$ yields
\begin{equation}\label{eq:per-step-proof}
  R_s
  \;:=\; \sum_{\zeta=(s-1)T+1}^{sT}\!\inp{P_s-P^{(\zeta)}}{x_\zeta x_\zeta^\top}
  \;\le\;
  \frac{d_{(s-1)T+1}^2 - \E[d_{sT+1}^2]}{2\eta}
  \;+\; 2kd\eta
  \;+\; 2k\,E_s,
\end{equation}
where $E_s:=\sum_{\zeta=(s-1)T+1}^{sT}\|E_\zeta\|_2$.

\paragraph{Across steps (drift linking).}
Summing \eqref{eq:per-step-proof} over $s=1$ to $S$ gives
\begin{equation}\label{eq:sum-steps-raw}
  \E[R_S]
  \;\le\;
  2kd\cdot\eta T
  \;+\;
  \frac{1}{2\eta}
  \sum_{s=1}^{S}\!\Big(d_{(s-1)T+1}^2 - \E[d_{sT+1}^2]\Big)
  \;+\;
  2k\,E_{\rm tot},
  \qquad
  E_{\rm tot}:=\sum_{s=1}^S E_s.
\end{equation}
The terminal terms are non-positive and can be dropped:
$-\sum_{s=1}^S \E[d_{sT+1}^2]\le 0$.
It remains to upper bound
\(
\sum_{s=1}^{S} d_{(s-1)T+1}^2
\)
where each $d_{(s-1)T+1}^2$ is measured against the \emph{local} comparator $V_s$.
Let $P_Q:=Q^{((s-1)T+1)}\!(Q^{((s-1)T+1)})^\top$.
By Lemma~\ref{lem:dist-identity} and Cauchy-Schwarz inequality,
\[
  \begin{aligned} &d_{(s-1)T+1}(V_s)^2 - d_{(s-1)T+1}(V_{s-1})^2 \\ &\quad
  \;=\;
  \tfrac12\!\left(\|P_Q-P_s\|_F^2 - \|P_Q-P_{s-1}\|_F^2\right) \\ &\quad
  \;=\;
  \inp{P_Q}{P_{s-1}-P_s}
  \;\le\; \sqrt{k}\,\lambda, \end{aligned}
\]
using $\|P_Q\|_F=\sqrt{k}$ and the drift assumption $\|P_s-P_{s-1}\|_F\le\lambda$.
With the initialization $P^{(1)}=0$ (hence $d_1(V_1)^2=\tfrac12\|P_1\|_F^2=\tfrac{k}{2}$), we obtain
\begin{equation}\label{eq:start-sum}
  \sum_{s=1}^{S} d_{(s-1)T+1}^2
  \;\le\;
  \frac{k}{2} \;+\; (S-1)\sqrt{k}\,\lambda .
\end{equation}

\paragraph{Conclusion.}
Plugging \eqref{eq:start-sum} into \eqref{eq:sum-steps-raw} yields
\[
  \E[R_S]
  \;\le\;
  2kd\cdot\eta ST
  \;+\;
  \frac{\frac{k}{2} + (S-1)\sqrt{k}\,\lambda}{2\eta}
  \;+\;
  2k\,E_{\rm tot},
\]
which is (7).
Optimizing the right-hand side over $\eta>0$ gives
\(
  \eta^\star=\sqrt{\frac{\frac{k}{2}+(S-1)\sqrt{k}\lambda}{2kd\cdot ST}}
\)
and
\[
  \E[R_S]
  \;\le\;
  \sqrt{2kdST\!\left(\tfrac{k}{2} + (S-1)\sqrt{k}\,\lambda\right)}
  \;+\;
  2k\,E_{\rm tot},
\]
which is (8).
\end{proof}

\paragraph{Remark (initialization).}
If one starts Oja from any orthonormal $Q^{(1)}$ (so $P^{(1)}$ is a rank-$k$ projector), then $d_1(V_1)^2\le k$ and the bound above holds with $\tfrac{k}{2}$ replaced by $k$; using $P^{(1)}=0$ in the analysis yields the slightly tighter $\tfrac{k}{2}$ constant used in the theorem.

\section{More discussions on related work}  \label{app:more-related-work}

\subsection{Low-Rank Training Paradigms}

\paragraph{Parameter weights subspace projection based descent.} Recent research has explored learned subspace descent methods, where gradient updates are confined to a low-dimensional subspace of the model’s weights that is itself learned during training. A straightforward instance is to parameterize each weight matrix as a product of low-rank factors (e.g. $W = B A$ with $B \in \mathbb{R}^{m \times r}, A \in \mathbb{R}^{r \times n}$) and optimize those factors via gradient descent \citep{hu2022lora, kamalakara2022exploring,lv2024scalable}. While such plain low-rank factorization can significantly reduce parameters, it often falls short of full-rank performance \citep{zhao2024galore}, prompting more adaptive approaches. ReLoRA \citep{lialin2023relora} tackles this by periodically absorbing the learned low-rank updates into the main weights and introducing new factor matrices, gradually increasing the effective rank of $W$ over training to recover expressiveness. SLTrain \citep{han2024sltrain} takes a complementary route by representing each weight as the sum of a trainable low-rank matrix and a sparse matrix ($W = B A + S$) and jointly optimizing both; this hybrid decomposition boosts model capacity while still confining updates to a learned low-dimensional (and partially sparse) subspace. Other methods learn to adjust the subspace itself: AdaLoRA \citep{zhang2023adalora} dynamically adapts the rank of the low-rank adapters based on an importance metric (simulating SVD pruning of negligible directions), effectively learning the most productive subspace basis during training. Likewise, LORO \citep{mo2025parameter} directly optimizes the model on a low-rank manifold (using Riemannian optimization), ensuring all gradient steps stay within a trainable low-dimensional parameter space. Beyond pure low-rank factors, structured subspace strategies have been proposed: BTT \citep{qiu2024compute} introduces a block tensor train factorization (via Kronecker-product structured weight matrices) with the factor components learned by gradient descent, and CoMERA \citep{yang2024comera} reshapes weight matrices into higher-order tensors that are decomposed into smaller learned tensors, enforcing a compact learned subspace of parameters. All of these approaches maintain that the projection basis (e.g. the low-rank factors or tensor components) is updated by training, allowing the subspace to evolve to fit the data.

In contrast, other methods restrict gradient descent to a fixed or partially fixed subspace of the weights, where the projection basis remains static (or changes only via non-gradient steps) for some duration. SwitchLoRA \citep{zhou2024switchlora} exemplifies this paradigm by using LoRA-style low-rank adapters but periodically switching their basis vectors (columns of $B$ or rows of $A$) with new ones during training, instead of continuously refining the same basis. This schedule of refreshing the adapter subspace (inspired by ReLoRA’s rank-growth idea) lets the model incrementally cover diverse low-dimensional subspaces without explicitly learning the basis via backpropagation. Similarly, LoQT \citep{loeschcke2024loqt} employs a fixed-basis approach for efficiency: it initializes each weight with two low-rank factors $W \approx B A$, then freezes one factor (keeping it constant and even quantized to save memory) and only trains the other factor $B$. Periodically, the product $B A$ is added into the full weight matrix $W$ (and possibly $B$ is reinitialized), but the subspace defined by the fixed factor $A$ remains unchanged during those intervals. By operating in a predetermined low-rank subspace (or a sequence of static subspaces), SwitchLoRA and LoQT avoid the overhead of learning new basis directions, instead confining updates to coefficient adjustments within a fixed projector. This ensures that gradient descent explores a low-dimensional slice of the parameter space defined a priori (or piecewise fixed), in contrast to the above methods that co-learn the subspace itself during training.

\paragraph{Gradient subspace projection descent.} Recent research has explored gradient subspace projection descent methods to reduce memory and computation by restricting updates to a low-dimensional subspace. GaLore \citep{zhao2024galore} pioneered this approach, using SVD to project each gradient onto a top-$r$ singular vector subspace. This low-rank update dramatically lowers memory usage for gradients and optimizer states. Numerous extensions build on GaLore’s foundation. Fira \citep{chen2024fira} adds an orthogonal complement term to reintroduce full-rank gradient information, closing the performance gap to full Adam without extra memory. Natural GaLore \citep{das2024natural} integrates natural gradient preconditioning into GaLore’s update, while Q-GaLore \citep{zhang2024q} combines low-rank projection with low-bit quantization for further memory savings. Other variants modify how the subspace is selected. COAP \citep{xiao2025coap} learns a correlation-aware projection matrix for improved efficiency, and GRASS \citep{muhamed2024grass} constructs sparse projection bases guided by gradient row norms to accelerate training. FRUGAL \citep{zmushko2024frugal} performs an Adam-style update in the projected subspace and a simultaneous SGD step on the orthogonal complement, enabling updates across the full parameter space with minimal extra memory overhead. Tensor-GaLore \citep{george2025tensor} extends the low-rank projection principle to higher-order tensor parameters (beyond flat matrices), broadening applicability while preserving memory benefits.

\paragraph{Theory.}
We frame the phenomenon of gradient subspace drift within the broader theoretical context of concept drift in online learning. Foundational work in this area establishes regret guarantees for learners adapting to a slowly shifting target distribution~\citep{crammer2010regret}. This theoretical lens is highly relevant to LLM training, where empirical evidence shows that gradients concentrate in a low-dimensional, slowly evolving subspace~\citep{gur2018gradient, larsen2021many}. While the idea of a shifting target is well-established, our work is the first to formally analyze a low-rank optimizer through the framework of dynamic regret under concept drift. This provides a principled theoretical foundation for Ori's performance and differentiates it from prior methods that focused primarily on algorithmic or empirical contributions.

Meanwhile, some works address the computational cost of maintaining the subspace or explore dynamic subspace strategies. Galore+ \citep{liao2024galore} reduces SVD overhead by performing the decomposition on only a randomly chosen component (e.g. one attention head) per update, significantly speeding up training at the cost of a slight approximation error. SubTrack-Grad \citep{rajabi2025subtrack++} and Online Subspace Descent \citep{liang2024memory} 
both dynamically update the projection basis: they solve auxiliary optimization problems (e.g. least-squares fits) to adjust the subspace so that it continues to approximate the current gradient well over time. In contrast, Flora \citep{wang2024flora} and GoLore \citep{he2024subspace} propose using fixed random projections of gradients, achieving memory savings with competitive performance to GaLore but without expensive SVD computations. I3S \citep{zhang2025i3s} introduces an importance sampling subspace selection scheme, weighting basis directions by their singular values to adaptively refine GaLore’s subspace and further improve convergence. Finally, the recent GaLore 2 framework \citep{su2025galore} addresses practical challenges like SVD cost and parallelization integration (e.g. with FSDP), and demonstrates the scalability of subspace-projected training by pre-training a 7B-parameter LLM with up to 500B tokens using significantly less memory.

\subsection{Online Subspace Tracking}
\paragraph{Online k-PCA subspace descent related method.} Recent optimizers exploit low-dimensional subspaces to reduce memory cost in large-scale training via online $k$-PCA. For example, Subspace-Momentum confines the momentum vector to a fixed $k$-dimensional subspace (e.g., spanned by top singular directions) while performing standard SGD in the orthogonal complement, substantially lowering optimizer-state memory and admitting high-probability convergence guarantees under standard conditions~\citep{nguyen2024lean}. The subspace can be chosen once or periodically refreshed via streaming singular vectors (e.g., computing top-$k$ components every $G$ steps) to adapt to evolving gradients~\citep{nguyen2024lean}. 
Similarly, Online Subspace Descent forgoes expensive SVD-based updates by continuously adapting the projection matrix through an online $k$-PCA algorithm~\citep{liang2024memory}. This approach provides a convergence guarantee for arbitrary subspace update rules in such low-rank optimizers and maintains an up-to-date basis with minimal overhead, eliminating full SVD computations during training~\citep{liang2024memory}. Applied to LLM pre-training, the online subspace optimizer achieves lower perplexity than prior low-rank methods and narrows the performance gap to full-rank training on benchmarks~\citep{liang2024memory}. 

\subsection{Concept Drift in Online Learning}
\paragraph{Subspace learning.} A growing body of work indicates that SGD dynamics concentrate in a low-dimensional, slowly evolving subspace, motivating algorithms that explicitly track or exploit subspace drift during training. Theory shows: (i) gradient updates lie within a tiny active subspace, suggesting that learning can be captured by a small set of dominant directions~\citep{gur2018gradient}; (ii) online learning with concept drift admits regret guarantees against a slowly shifting comparator, formalizing performance when the target (or optimal subspace) moves over time~\citep{crammer2010regret}; and (iii) low-dimensional training succeeds once the subspace dimension exceeds a threshold tied to the loss landscape’s Gaussian width, explaining why few degrees of freedom can suffice and how this threshold depends on initialization~\citep{larsen2021many}. On the practical side for large-scale LLMs, recent methods instantiate drifting-subspace training by refreshing projection bases during optimization; among them, I3S~\citep{zhang2025i3s} samples singular directions by importance to diversify consecutive subspaces, alleviating “frozen” dominant bases while retaining convergence guarantees comparable to random-projection methods and delivering stronger pretraining perplexity in practice. 
Together, these results motivate modeling—and algorithmically tracking—gradient subspace drift when designing memory-efficient, very low-rank training procedures for LLMs.

\section{Ori algorithm details}  \label{app:algorithm-details}

\subsection{AdamW with Ori}

\Cref{alg:main_loop_detailed} shows the full algorithm of an AdamW optimizer equipped with Ori low-rank gradient projector updates. We note that Ori is not intrinsically entangled with AdamW; rather similar to GaLore~\citep{zhao2024galore}, the Ori projector update subroutine can be easily integrated with other optimizers. We leave such exploration to future work.

\begin{algorithm}[t]
\caption{AdamW with Ori (per-layer update)}
\label{alg:main_loop_detailed}
\begin{algorithmic}[1]
\REQUIRE Layer weights $W \in \mathbb{R}^{m \times n}$
\REQUIRE Learning rate $\eta$, AdamW hyperparameters $\beta_1, \beta_2$, numerical constant $\epsilon$
\REQUIRE Rank $k$, scale factor $\gamma$
\REQUIRE Ori hyperparameters $T_{\text{warmup}}$, $T_{\text{period}}$, Oja++ settings $H_{\text{oja}}$

\ENSURE Updated weights $W$

\STATE Initialize $Q \in \mathbb{R}^{m \times k}$ randomly
\STATE $P \gets \mathrm{QR}(Q)$ \hfill\COMMENT{$P$ has orthonormal columns}
\STATE Initialize first moment $M \gets 0 \in \mathbb{R}^{k \times n}$
\STATE Initialize second moment $V \gets 0 \in \mathbb{R}^{k \times n}$
\STATE $t \gets 0$

\REPEAT
    \STATE $G \gets \nabla_{W}\mathcal{L}(W)$ \hfill\COMMENT{Full gradient}

    \IF{$t < T_{\text{warmup}}$ \OR $t \bmod T_{\text{period}} = 0$}
        \STATE $Q \gets \textsc{UpdateProjector}(Q, G, H_{\text{oja}})$ \hfill\COMMENT{Ori/Oja++ update}
        \STATE $P \gets \mathrm{QR}(Q)$
    \ENDIF

    \STATE $R \gets P^\top G$ \hfill\COMMENT{Projected gradient in $\mathbb{R}^{k \times n}$}

    \STATE $M \gets \beta_1 M + (1-\beta_1)R$
    \STATE $V \gets \beta_2 V + (1-\beta_2)(R \odot R)$ \hfill\COMMENT{$\odot$: elementwise square}
    \STATE $\hat{M} \gets M / (1-\beta_1^{t+1})$ \hfill\COMMENT{Bias correction}
    \STATE $\hat{V} \gets V / (1-\beta_2^{t+1})$
    \STATE $N \gets \hat{M} \oslash (\sqrt{\hat{V}} + \epsilon)$ \hfill\COMMENT{$\oslash$: elementwise divide}

    \STATE $\tilde{G} \gets \gamma \, (P N)$ \hfill\COMMENT{Lift back to $\mathbb{R}^{m \times n}$}
    \STATE $W \gets W - \eta \tilde{G}$
    \STATE $t \gets t + 1$
\UNTIL{convergence criteria met}

\STATE \textbf{return} $W$
\end{algorithmic}
\end{algorithm}

\subsection{Ori Hyperparameters and Recommended Search Strategy}

In total, Ori consists of the following 7 important hyperparameters, listed roughly in the order of the most to least influential to model performance as we observed during method development.
\begin{itemize}
    \item Scale $\gamma$: Similar to GaLore's scale $\alpha$~\citep{zhao2024galore}, this parameter effectively scales the actual learning rate for the model layers for which Ori project the gradients. We suggest always search over this parameter as it highly correlates with model performance. Empirically, we found that one can set a higher $\gamma$ value in the extremely low rank settings, while smaller $\gamma$ tend to yield more stable training at higher ranks. We suggest $\gamma \in \{ 0.25, 0.5, 0.75, 1.0 \}$ for large model pre-training, and $\gamma \in \{ 2.0, 4.0 \}$ for very low-rank fine-tuning.
    \item Drift coefficient $\rho$: This parameter controls how much the last step's projector is carried over to warm-start the next step's projector Oja updates. A $\rho=0.0$ means total random restart, and a $\rho=1.0$ means total carry-over with no randomness injected. As established in \Cref{sec:exp-ablation}, the $\rho=0.99$ tend to be the universally good setting. We also found that $\rho \in \{ 0.8, 0.9 \}$ can be good alternatives in a few scenarios. 
    \item Projector update period $T_{\text{period}}$: This parameter controls the frequency of Ori's update. As established in \Cref{sec:exp-ablation}, we suggest $T_{\text{period}}=20$ as it strikes a perfect balance between model performance and computation budget.
    \item Ori warmup period $T_{\text{warmup}}$: This parameter controls for how many steps at the begging does Ori ignore $T_{\text{period}}$ and update its projector consecutively. We recommand avoid setting too large a value as it may lead to overfitting. Empirically, we observe that $T_{\text{warmup}} = 20$ tend to be stable.
    \item Initial Oja learning rate $\eta$: The initial learning rate for the Oja update (\Cref{eq:oja-update}), which gets decreased by a cosine annealing scheduler, and then gets reset at the next model update step. We observe that, as long as $\eta$ is not too small, Ori yields stable learning and $\eta$'s specific value do not have too much influence. We recommend setting $\eta = 1.0$ for all scenarios.
    \item Number of inner Oja-loop updates $N_{\text{update}}$: This parameter controls the number of inner Oja updates (see \Cref{alg:update_projector}). Note that since Ori uses Oja++ outer loop, the total number of inner Oja updates depends on the optimizer rank: $(\log_2(k) + 1) N_{\text{update}}$. In practice, we set $N_{\text{update}} = 2$ for all experiments.
    \item Number of gradient columns to sample for Oja update $N_{\text{vecs}}$: In practice, we observe that setting it in a way that scales with the model's embedding dimension is a good choice. We suggest $N_{\text{vecs}} \in \{ d_{\text{model}} / 2, d_{\text{model}}, d_{\text{model}} \times 2 \}$.
\end{itemize}

\subsection{Amortized Complexity Analysis}
\label{app:complexity}

In this section, we analyze the amortized complexity of Ori update steps as a function of the hyperparameter knobs as described in \Cref{alg:main_loop_detailed}.

\paragraph{Setup.} Let $d \times n$ be the weight matrix dimensions, $r$ the rank, $T_{\text{period}}$ the update interval, $N_{\text{updates}}$ the inner Oja iterations per block stage, $N_{\text{vecs}}$ the number of sampled gradient columns, and $B = \lceil \log_2(r) \rceil$ the number of Oja++ block stages.

\paragraph{Per-step cost (every step, lines 12–19 of \Cref{alg:main_loop_detailed}).} Projecting the gradient $R = P^\top G$ and lifting back $\tilde{G} = \gamma(PN)$ each cost $O(dnr)$. The Adam update in the projected space costs $O(rn)$. Thus:
\begin{equation*}
    C_{\text{step}} = O(dnr) .
\end{equation*}

\paragraph{Per-update cost (every $T_{\text{period}}$ steps, \Cref{alg:update_projector}).} The Oja++ outer loop runs $B$ block stages with growing active slice sizes, each performing $N_{\text{updates}}$ inner iterations. Each inner iteration costs $O(d \cdot N_{\text{vecs}} \cdot r_{\text{next}}(k))$, where $r_{\text{next}}(k)$ is the active slice at stage $k$. QR re-orthogonalization costs $O(dr^2)$ per update. Since $\sum_{k=0}^{B-1} r_{\text{next}}(k) = O(rB)$:
\begin{equation*}
    C_{\text{update}} = O(B \cdot N_{\text{updates}} \cdot d \cdot N_{\text{vecs}} \cdot r + dr^2) .
\end{equation*}

\paragraph{Amortized per-training-step cost.}
\begin{equation*}
    C_{\text{amortized}} = O(dnr) + O\left(\frac{B \cdot N_{\text{updates}} \cdot d \cdot N_{\text{vecs}} \cdot r + dr^2}{T_{\text{period}}}\right) .
\end{equation*}

In our default configuration ($N_{\text{updates}} = 2$, $T_{\text{period}} = 20$), the amortized update overhead scales as $O\left(\frac{B \cdot N_{\text{vecs}} \cdot dr}{10}\right)$, which is small relative to $C_{\text{step}} = O(dnr)$ when $B \cdot N_{\text{vecs}} / 10 \ll n$, which is easily satisfied in practice (e.g., $B=7$, $N_{\text{vecs}}=512$, $n=2048$ gives ratio $\approx 0.17$).

\paragraph{Memory per layer.} Ori stores the projector $P \in \mathbb{R}^{d \times r}$ and Adam moments $M, V \in \mathbb{R}^{r \times n}$, totaling $O(r(d + 2n))$. At $r=128$, $d=n=2048$, this is ~9.4\% of full-rank Adam's $O(2dn)$, as confirmed later by \Cref{tab:memory-time} ($0.2506$ GB for Ori vs $2.4942$ GB for Adam).

Despite updating $10\times$ more frequently, Ori is the fastest gradient-subspace method. This pattern holds across all model sizes and ranks (as shown in \Cref{tab:memory-time-full}).

\section{Stochastic Power Iteration for LDAdam}\label{app:oLDAdam}
LDAdam~\citep{robert2024ldadam} is a memory-efficient variant of Adam for low-rank LLM training that performs adaptive updates inside a low-dimensional gradient subspace, while using projection-aware rules to transfer the first/second-moment states of Adam as the subspace changes over time. It computes the subspace via the warm-started block power iteration method (warm-started by SVD library call at the very first update step) and adds a generalized error-feedback mechanism to compensate for information lost when compressing both gradients and the first/second-moment optimizer states. As a result, it approaches Adam-like behavior with much smaller optimizer memory in the low-rank setting. 

A key practical limitation, however, is that LDAdam typically requires per-step projector updates (T=1) to maintain accuracy, which can be computationally prohibitive. 
On the other hand, naively increasing LDAdam's subspace update interval to $T=200$ -- matching the standard schedule used by GaLore -- causes a significant performance drop. This is shown in \Cref{tab:pretrain-c4-ldadam}, our main results for pretraining the Llama model on C4 dataset. 
In particular, on the largest, 1b-size model, we find that the training process under LDAdam with $T=200$ may fail to converge, leading to very large evaluation perplexity (above 160 perplexity for both rank $r=512$ and $r=128$). 

We attribute this sensitivity and instability to LDAdam’s reliance on a single-step full-block power-iteration-style refresh: when the gradient subspace drifts over long intervals (e.g., $T=200$), a one-step update is often insufficient to re-align the projector.

Motivated by Ori’s online subspace tracking, we replace LDAdam’s single-step, full-block power iteration with a multi-step stochastic procedure. Concretely, at each step $t$, given the input matrix $M$ from which we intend to compute the subspace basis, we 
sample a row or column vector $x_t$ from $M$, then update the basis via
\begin{equation} \label{eq:stochasticPI-update}
Q_t \leftarrow \text{QR}\left( (\eta_t x_t x_t^\top) Q_{t-1} \right),
\end{equation}
where $\eta_t$ is the update rate at step $t$.
This update can be viewed as a special case of Ori’s Oja-style rule with an explicit identity “carry-over” term,
\begin{equation} \label{eq:oja-alpha-update}
Q_t \leftarrow \text{QR}\left( (\alpha I + \eta_t x_t x_t^\top) Q_{t-1} \right),
\end{equation}
where stochastic power iteration corresponds to setting $\alpha = 0$. We refer to the resulting optimizer as O-LDAdam, reflecting the integration of Ori-style online subspace tracking into LDAdam’s projection-aware, error-correcting framework. \Cref{alg:ldadam-galore-style} presents the algorithm pseudocode of O-LDAdam.

\begin{algorithm}
\caption{O-LDAdam}
\label{alg:ldadam-galore-style}
\begin{algorithmic}[1]
\STATE \textbf{for} weight \textbf{in} model.parameters():
\STATE \hspace{1.6em} grad $\leftarrow$ weight.grad
\STATE \hspace{1.6em} \# states: basis $U$, moments $(m,v)$, error buffer $\xi$
\STATE \hspace{1.6em} $A \leftarrow$ grad $+ \xi$ \hfill \COMMENT{unload error buffer}
\STATE \hspace{1.6em} $U_{\text{old}} \leftarrow U$
\STATE \hspace{1.6em} \textbf{for} $i = 1$ to $N_{\text{steps}}$: \hfill \COMMENT{stochastic subspace tracking}
\STATE \hspace{3.2em} $x_i \leftarrow \texttt{SampleColumn}(A)$
\STATE \hspace{3.2em} $U \leftarrow \text{QR}\left( (\alpha I + \eta_i x_i x_i^\top) U \right)$ \hfill \COMMENT{$\alpha=0$ for stochastic PI}
\STATE \hspace{1.6em} \textbf{end for}
\STATE \hspace{1.6em} $(m,v) \leftarrow \texttt{transfer\_states}(U_{\text{old}} \to U, m, v)$ \hfill \COMMENT{proj-aware transfer}
\STATE \hspace{1.6em} a $\leftarrow U^\top A$ \hfill \COMMENT{original $\to$ compact}
\STATE \hspace{1.6em} $\Delta a, (m,v) \leftarrow \texttt{AdamUpdate}(a, m, v)$ \hfill \COMMENT{update in compact space}
\STATE \hspace{1.6em} $\Delta \leftarrow U\,\Delta a$ \hfill \COMMENT{compact $\to$ original}
\STATE \hspace{1.6em} weight.data $\leftarrow$ weight.data $+ \Delta$
\STATE \hspace{1.6em} $\xi \leftarrow \texttt{error\_feedback}(A, U_{\text{old}}, U, m)$ \hfill \COMMENT{compensate compression}
\end{algorithmic}
\end{algorithm}

\section{Experiment details}
\label{app:exp-details}

\subsection{Dataset and Training Details}
\label{app:dataset-training-details}

We use the C4 dataset for the LLaMA model pre-training task. C4 dataset is a colossal, cleaned version of Common Crawl's web crawl corpus, mainly intended for pre-training language models and word representations, available under the Open Data Commons Attribution License (ODC-BY) license~\citep{raffel2020exploring}. Similar to prior work's experiment settings, we train the model without data repetition over a sufficiently large amount of data. We report the final validation perplexity (the lower the better) for the main experiment table (\Cref{tab:pretrain-c4}).

For all pre-training experiments, we adjust the gradient accumulation steps accordingly, depending on the available GPUs, so that the effective total batch size for all runs are 512 (for instance, an actual batch size of 64 with 8 gradient accumulation steps results in the effective batch size of 512). We set the LLaMA model max length to 256, adopted a cosine learning rate scheduler with a 0.1 minimum learning rate ratio, and trained with bfloat16 data type.

We use the GLUE dataset for the Roberta-Base model fine-tuning task. GLUE is a benchmark for evaluating the performance of NLP models on a variety of tasks, including question answering, sentiment analysis, and textual entailment~\citep{wang2018glue}. We report the final accuracy score (the higher the better) for the main experiment table (\Cref{tab:finetune-glue}). The metric of choice for each GLUE task is listed in \Cref{tab:glue-metric}. Notably, different from many of the prior optimization works, we run each task + optimizer + rank combination for three runs with different random seeds to improve statistical confidence level. Thus, we were able to report the error bar (standard deviation) in the table.

For all the fine-tuning experiments, similar to prior work's experiment settings, we train all models for a total of 30 epochs, with 1 gradient accumulation step. We set the model max length to 512, and adopted a linear learning rate scheduler. 

\begin{table}[t]
    \centering
    \caption{Accuracy metrics we used for the Roberta-Base fine-tuning experiments of \Cref{tab:finetune-glue}.}
    \resizebox{1.0\textwidth}{!}{
    \begin{tabular}{l cccccccc}
        \toprule
        Task   & CoLA & STS-B & MRPC & RTE & SST-2 & MNLI & QNLI & QQP \\
        \midrule
        Metric & Matthews Correlation & Pearson & F1 & Accuracy & Accuracy & Accuracy & Accuracy & F1 \\
        \bottomrule
    \end{tabular}
    }
    \label{tab:glue-metric}
\end{table}

\subsection{Compute resources}
\label{app:compute-resources}

We used a variety of compute resources with different GPU hardware, including RTX 4090, RTX 4070 Ti Super, RTX A500, RTX A6000, 40GB A100, 80GB A100, H100, MI100, and MI210. For the large model pre-training, such as the 350M and 1B LLaMA model pre-training experiments, we trained the model on nodes equipped with 4 H100 or 4 MI210.

\subsection{Ori hyperparameters}
\label{app:ori-hyperparameters}

We report the detailed Ori hyperparameters used in each experiment setting for reproducibility. \Cref{tab:ori-hyper-pretrain} lists all hyperparameters used for the LLaMA-C4 pre-trianing tasks, and \Cref{tab:ori-hyper-finetune} lists all hyperparameters used for the Roberta-Base-GLUE fine-tuning tasks.

\begin{table}[t] 
\centering
\caption{Ori hyperparameter used for the LLaMA pre-training results on the C4 dataset (\Cref{tab:pretrain-c4}).}
\label{tab:ori-hyper-pretrain}
\resizebox{\textwidth}{!}{
\begin{tabular}{l rrr rrr rrr rrr}
\toprule
\multicolumn{1}{c}{\multirow{2}{*}{\textbf{Ori Hyperparameters}}}& \multicolumn{3}{c}{\textbf{60M}} & \multicolumn{3}{c}{\textbf{130M}} & \multicolumn{3}{c}{\textbf{350M}} & \multicolumn{3}{c}{\textbf{1B}} \\
\cmidrule(lr){2-4} \cmidrule(lr){5-7} \cmidrule(lr){8-10} \cmidrule(lr){11-13}
 & $r=128$ & $r=32$ & $r=8$ & $r=256$ & $r=64$ & $r=16$ & $r=256$ & $r=64$ & $r=16$ & $r=512$ & $r=128$ & $r=32$ \\
\midrule
$\gamma$            & 0.5 & 0.5 & 1.0    & 0.25 & 0.25 & 0.25 & 0.75 & 0.75 & 0.5 & 0.5 & 0.5 & 0.5 \\
$\rho$              & 0.99 & 0.99 & 0.99 & 0.99 & 0.99 & 0.99 & 0.99 & 0.99 & 0.99 & 0.99 & 0.99 & 0.99 \\
$T_{\text{period}}$ & 20  & 20  & 20     & 20  & 20  & 20     & 20  & 20  & 20 & 20  & 20  & 20 \\
$T_{\text{warmup}}$ & 20  & 20  & 20     & 20  & 20  & 20     & 20  & 20  & 20 & 20  & 20  & 20 \\
$\eta$              & 1.0 & 1.0 & 1.0    & 1.0 & 1.0 & 1.0    & 1.0 & 1.0 & 1.0  & 1.0 & 1.0 & 1.0  \\
$N_{\text{update}}$ & 2   & 2   & 2      & 2   & 2   & 2      & 2   & 2   & 2  & 2   & 2   & 2  \\
$N_{\text{vecs}}$   & 256 & 256 & 256    & 768 & 768 & 768    & 256 & 256 & 256 & 512 & 512 & 512   \\
\bottomrule
\end{tabular}
} 
\end{table}

\begin{table}[t]
    \centering
    \caption{Ori hyperparameter used for the Roberta-Base fine-tuning results on the GLUE dataset (\Cref{tab:finetune-glue})}
    \label{tab:ori-hyper-finetune}
    \begin{tabular}{c c}
        \toprule
        \textbf{Ori Hyperparameters} & \textbf{GLUE all tasks} \\
        \midrule
        $\gamma$            &  4.0 \\
        $\rho$              &  0.99 \\
        $T_{\text{period}}$ &  20  \\
        $T_{\text{warmup}}$ &  20 \\
        $\eta$              &  1.0 \\
        $N_{\text{update}}$ &  2 \\
        $N_{\text{vecs}}$   &  768 \\
        \bottomrule 
    \end{tabular} 
\end{table}

\subsection{O-LDAdam hyperparameters}
\label{app:oldadam-hyperparameters}

We report the subset of O-LDAdam hyperparemeters pertinent to the Ori-inspired stochastic power iteration update in each experiment setting for reproducibility. The rest of O-LDAdam hyperparameters are inherited from, and thus the same as, the baseline LDAdam\citep{robert2024ldadam}. \Cref{tab:o-ldadam-hyper-pretrain} lists all hyperparameters used for the LLaMA-C4 pretraining tasks, and \Cref{tab:o-ldadam-hyper-finetune} lists all hyperparameters used for the Roberta-Base-GLUE fine-tuning tasks.

\begin{table}[t] 
\centering
\caption{O-LDAdam hyperparameter related to the Ori-inspired stochastic power iteration update used for the LLaMA pre-training results on the C4 dataset (\Cref{tab:pretrain-c4-ldadam}).}
\label{tab:o-ldadam-hyper-pretrain}
\resizebox{\textwidth}{!}{
\begin{tabular}{l rrr rrr rrr rrr}
\toprule
\multicolumn{1}{c}{\multirow{2}{*}{\textbf{Ori Hyperparameters}}}& \multicolumn{3}{c}{\textbf{60M}} & \multicolumn{3}{c}{\textbf{130M}} & \multicolumn{3}{c}{\textbf{350M}} & \multicolumn{3}{c}{\textbf{1B}} \\
\cmidrule(lr){2-4} \cmidrule(lr){5-7} \cmidrule(lr){8-10} \cmidrule(lr){11-13}
 & $r=128$ & $r=32$ & $r=8$ & $r=256$ & $r=64$ & $r=16$ & $r=256$ & $r=64$ & $r=16$ & $r=512$ & $r=128$ & $r=32$ \\
\midrule
$T_{\text{period}}$ & 200  & 200  & 200     & 200  & 200  & 200     & 200  & 200  & 200 & 200  & 200  & 200 \\
$T_{\text{warmup}}$ & 0  & 0  & 0     & 0  & 0  & 0     & 0  & 0  & 0 & 0  & 0  & 0 \\
$\eta$              & 1.0 & 1.0 & 1.0    & 1.0 & 1.0 & 1.0    & 1.0 & 1.0 & 1.0  & 1.0 & 1.0 & 1.0  \\
$N_{\text{steps}}$ & 8   & 8   & 8      & 8   & 8   & 8      & 16   & 16   & 16  & 16   & 16  & 16  \\
$N_{\text{vecs}}$   & 1 & 1 & 1    & 1 & 1 & 1    & 1 & 1 & 1 & 1 & 1 & 1   \\
\bottomrule
\end{tabular}
}
\end{table}

\begin{table}[t]
    \centering
    \caption{O-LDAdam hyperparameter related to the Ori-inspired stochastic power iteration update used for the Roberta-Base fine-tuning results on the GLUE dataset (\Cref{tab:finetune-glue})}
    \label{tab:o-ldadam-hyper-finetune}
    \resizebox{\textwidth}{!}{
    \begin{tabular}{c c c}
        \toprule
        \textbf{Ori Hyperparameters} & \textbf{GLUE task MNLI, MRPC, CoLA, QQP, RTE, STSB} & \textbf{GLUE task SST2, QNLI} \\
        \midrule
        $T_{\text{period}}$ &  200  & 200 \\
        $T_{\text{warmup}}$ &  0    & 0\\
        $\eta$              &  1.0  & 1.0 \\
        $N_{\text{steps}}$  &  4    & 2 \\
        $N_{\text{vecs}}$   &  1    & 1 \\
        \bottomrule 
    \end{tabular} 
    }
\end{table}

\subsection{Baseline hyperparameters}
\label{app:baseline-hyperparameters}

Baseline hyperparameters are discussed in the main experiment section (see \Cref{sec:exp-baseline}).

\subsection{Visualizing subspace drift}
\label{app:visualizing-subspace-drift}

We reused the setting of 60M LLaMA model pre-training on C4 dataset to visualize an actual gradient subspace drift as the model is being trained in real life, resulting in \Cref{fig:subspace-overlap}. 

For this experiment, we trained the 60M LLaMA model with the Adam optimizer for a reduced 1000 training steps, totalling 100M training tokens. We set $\Delta t = 10$, $k \in \{ 128, 32, 8 \}$, and computed $\text{overlap}\left(U^{(t)}_k, U^{(t+\Delta t)}_k\right)$ (\Cref{eq:subspace-overlap}) for every layers in the model. We then report and plot the mean overlap, averaged over all layers, through the training trajectories. 

To further demonstrate that the overlap values shown in \Cref{fig:subspace-overlap} are not ad-hoc and it indeed signify that the gradient subspace has non-trivial overlapping, we simulated three random baselines, one for each rank $k\in \{ 128, 32, 8 \}$. The random baseline is conducted by sampling a random matrix with the same shape and Frobenius norm as the true gradient matrix, for each layer at each step. We then pretend these random matrices were some ``gradients'', and record their overlap values with the earlier random matrices.

\subsection{Ablation study details}
\label{app:ablation-details}

We reused the setting of 60M LLaMA model pre-training on C4 dataset to perform the ablation study. The ori optimizer in the ablation studies used the same hyperparameters as detailed in \Cref{tab:ori-hyper-pretrain}. 

For the ablation of Oja++ vs. Standard Oja, recall that the default choice of Oja++ algorithm consists of an outer loop and an inner loop \Cref{alg:update_projector}, and the total number of Oja inner updates in the default setting is $(\log_2(k) + 1) N_{\text{update}}$. Thus, to ensure fairness, we set this number to the Oja-variant of Ori. That is, since we use $N_{\text{update}}=2$ for all experiments, in this ablation study, we set the number of Oja baseline udpates to 16 when $k=128$, 12 when $k=32$, and finally 8 when $k=8$.

\subsection{Measuring optimizer state peak memory}
\label{app:exp-memory-profiling}

Finally, we report our procedure for measuring the memory footprint of the optimizers evaluated in our experiments. Our goal is to record the peak memory usage, incurred solely by the said optimizer's update step. To this end, we take advantage of PyTorch's memory recording functionalities.

The below code listing shows the pseudocode for recording peack memory usage incurred solely by the optimizer's \texttt{step()} operations. We use this procedure to obtain all the results in \Cref{tab:memory-time,tab:memory-time-full}.

\begin{lstlisting}[language=Python, caption={Python code for profiling optimizer peak memory.}, label={lst:memory-code}]
import torch

# Model and optimizer setup
model = ...
optimizer = ...

# Main training loop
for batch in enumerate(dataloader):

    # Forward and backward pass
    loss = model(**batch).loss
    loss.backward()

    # Record the current memory usage after model backward
    mem_after_backward = torch.cuda.memory_allocated()

    # Reset the max memory counter
    torch.cuda.reset_peak_memory_stats()

    # Optimizer step
    optimizer.step()

    # Read the optimizer-incurred max memory
    # The difference with mem_after_backward is the 
    # maximum additional memory caused by optimizer.step()
    peak_mem_during_step = torch.cuda.max_memory_allocated()
    optim_peak_mem = peak_mem_during_step - mem_after_backward

\end{lstlisting}

\section{Additional experiment results}  \label{app:additional-experiment-results}

\subsection{Memory and Computation Time}
\label{sec:exp-memory-time}

\begin{table}[t]
\centering
\caption{Memory and Wallclock Time Comparison between Ori and non-error-feedback baselines (Adam, ReLoRA, GaLore, OSD). Analysis is conducted by pre-training 1B LLaMA on C4 dataset.}
\label{tab:memory-time}
\resizebox{0.6\linewidth}{!}{%
\begin{tabular}{l rrrr}
\toprule
Optimizer rank                                  & $r = 512$ & $r = 128$ & $r = 32$  \\
\midrule
\multicolumn{4}{c}{\textit{Optimizer State Peak Memory (GB)}} \\
\midrule
Adam                                            & 2.4942    & 2.4942    & 2.4942    \\
ReLoRA$^\text{\tiny\citep{lialin2023relora}}$   & 1.1349    & 0.4670    & 0.3000    \\
GaLore$^\text{\tiny\citep{zhao2024galore}}$     & 0.3385    & 0.3385    & 0.3385    \\
OSD$^\text{\tiny\citep{liang2024memory}}$       & 0.2832    & 0.2832    & 0.2832    \\
\textbf{Ori (Ours)}                             & {0.2583}	& {0.2506}	& {0.2487} \\
\midrule
\multicolumn{4}{c}{\textit{Average Update Time (s / Step)}} \\
\midrule
Adam                                            & 12.6015   & 12.6015	& 12.6015 \\
ReLoRA$^\text{\tiny\citep{lialin2023relora}}$   & 13.9454	& {11.9193}	& {11.5661} \\
GaLore$^\text{\tiny\citep{zhao2024galore}}$     & 12.9699	& 12.9703	& 13.0401 \\
OSD$^\text{\tiny\citep{liang2024memory}}$       & 12.8438	& 12.8373	& 12.8594 \\
\textbf{Ori (Ours)}                             & {12.7266}	& 12.7206 & 12.6705 \\
\bottomrule
\end{tabular}
}
\end{table}

\begin{table}[t]
\centering
\caption{Memory and Wallclock Time Comparison between O-LDAdam and an error-feedback baseline, LDAdam. Analysis is conducted by pre-training 1B LLaMA on C4 dataset.}
\label{tab:memory-time-ldadam}
\resizebox{0.6\linewidth}{!}{
\begin{tabular}{l rrrr}
\toprule
Optimizer rank                                  & $r = 512$ & $r = 128$ & $r = 32$  \\
\midrule
\multicolumn{4}{c}{\textit{Optimizer State Peak Memory (GB)}} \\
\midrule
LDAdam ($T=1$)                                  & 0.1221    & 0.1221    & 0.1221 \\
LDAdam ($T=200$)                                & 0.1221    & 0.1221    & 0.1221 \\
\textbf{O-LDAdam (Ours)}                      & 0.1221    & 0.1221    & 0.1221 \\
\midrule
\multicolumn{4}{c}{\textit{Average Update Time (s / Step)}} \\
\midrule
LDAdam ($T=1$)                                  & 13.9294   & 13.3475   & 13.1410 \\
LDAdam ($T=200$)                                & 12.7141   & 12.6996   & 12.6863 \\
\textbf{O-LDAdam (Ours)}                      & 13.4735   & 12.9018   & 12.8648 \\
\bottomrule
\end{tabular}
}
\end{table}

We profile the optimizer state peak memory and wall-clock time during the pre-training experiments from \Cref{sec:exp-pretrain}.
Details on the memory profiling procedure can be found in \Cref{app:exp-memory-profiling}.
As shown in \Cref{tab:memory-time}, Ori exhibits a superior efficiency profile among standard, non-error-feedback gradient-subspace methods (Adam, ReLoRA, GaLore, OSD). In addition, as shown in \Cref{tab:memory-time-ldadam}, O-LDAdam demonstrates competitive and comparable efficiency profile to the error-feedback baseline LDAdam under the large update frequency regime ($T=200$).
Full table on pre-training of other model sizes can be found in \Cref{app:additional-experiment-results}.

\textbf{Memory Usage:} Compared to standard, non-error-feedback baselines, Ori consistently consumes the least optimizer peak memory across all ranks. For instance, at rank 32, Ori requires only 0.2487 GB, which is 12.1\% less than OSD (0.2832 GB) and 26.5\% less than GaLore (0.3385 GB). This confirms Ori's tangible memory savings. 

On the other hand, compared to the error-feedback methods, O-LDAdam consistently yields the same memory footprint as LDAdam baselines.

\textbf{Training Speed:} Despite updating its subspace projector 10x more frequently ($T_{\text{Ori}}=20$ vs. $T_{\text{GaLore/OSD}}=200$), Ori is still the fastest among all standard, non-error-feedback methods. While the weight-subspace method ReLoRA is faster at lower ranks due to its different update mechanism, its worse model performance (see \Cref{tab:pretrain-c4}) highlights a clear trade-off. In contrast, Ori achieves superior model performance without compromising on speed, making it a more effective and efficient choice overall. The lightweight nature of the Oja++ update allows Ori to be both more accurate and faster than its direct competitors, making it a better choice for efficient training.

On the other hand, although O-LDAdam yields a slightly longer average update time than LDAdam with $T=200$ due to the need to perform power iterations multiple times (\Cref{alg:ldadam-galore-style}), the additional cost is not prohibitively large. Moreover, O-LDAdam demonstrates a nontrivial speedup compared to the standard LDAdam setting with a per-step update schedule ($T=1$). Together with the model performance results shown in \Cref{tab:pretrain-c4-ldadam}, this demonstrates that O-LDAdam strikes a good balance between model performance and computational cost in the extremely low-rank regime.

\Cref{tab:memory-time-full} presents the full results of memory and Wallclock time comparison of all the optimizer methods on the LLaMA model pre-training task on the C4 dataset. As shown by the table, our method Ori achieves superior memory and computation time profile almost across the broad, demonstrating its general efficiency under various settings.

\begin{table}[t] 
\centering
\caption{Memory and Wallclock Time Comparison of pre-training various sized LLaMa model on the C4 dataset.}
\label{tab:memory-time-full}
\resizebox{\textwidth}{!}{%
\begin{tabular}{l rrr rrr rrr rrr}
\toprule
\multicolumn{1}{c}{\multirow{2}{*}{\textbf{Method}}}& \multicolumn{3}{c}{\textbf{60M}} & \multicolumn{3}{c}{\textbf{130M}} & \multicolumn{3}{c}{\textbf{350M}} & \multicolumn{3}{c}{\textbf{1B}} \\
\cmidrule(lr){2-4} \cmidrule(lr){5-7} \cmidrule(lr){8-10} \cmidrule(lr){11-13}
 & $r=128$ & $r=32$ & $r=8$ & $r=256$ & $r=64$ & $r=16$ & $r=256$ & $r=64$ & $r=16$ & $r=512$ & $r=128$ & $r=32$ \\
\midrule
\multicolumn{13}{c}{\textit{Optimizer Peak Memory (GB)}} \\
\cmidrule(lr){1-13} %
Adam    & 0.1082	& -- & --	& 0.2498	& --	& --	& 0.6854	& --	& --	& 2.4942	& --	& -- \\
ReLoRA$^\text{\tiny\citep{lialin2023relora}}$  & 0.0797	& 0.0657	& 0.0622	& 0.1751	& 0.1125	& 0.0968	& 0.3450	& 0.1779	& 0.1361	& 1.1349	& 0.4670	& 0.3000 \\
GaLore$^\text{\tiny\citep{zhao2024galore}}$   & 0.0670	& 0.0670	& 0.0670	& 0.1049	& 0.1049	& 0.1049	& 0.1458	& 0.1458	& 0.1458	& 0.3385	& 0.3385	& 0.3385 \\
OSD$^\text{\tiny\citep{liang2024memory}}$     & 0.0635	& 0.0635	& 0.0635	& 0.0970	& 0.0970	& 0.0970	& 0.1318	& 0.1318	& 0.1318	& 0.2832	& 0.2832	& 0.2832 \\
\textbf{Ori (Ours)} 
        & \best{0.0623}	& \best{0.0619}	& \best{0.0618}	& \best{0.0948}	& \best{0.0932}	& \best{0.0927}	& \best{0.1256}	& \best{0.1237}	& \best{0.1232}	& \best{0.2583}	& \best{0.2506}	& \best{0.2487} \\
\bottomrule
\multicolumn{13}{c}{\textit{Average Update Time (s / Step)}} \\
\cmidrule(lr){1-13} %
Adam    & 0.5198	& --	& --	& 1.0309	& --	& --	& 2.655	& --	& --	& 12.6015	& --	& -- \\
ReLoRA$^\text{\tiny\citep{lialin2023relora}}$  & 0.6622	& 0.6395	& 0.6413	& 1.3324	& 1.2824	& 1.2626	& 3.3348	& 3.1886	& 3.1423	& 13.9454	& \best{11.9193}	& \best{11.5661} \\
GaLore$^\text{\tiny\citep{zhao2024galore}}$   & 0.5454	& 0.5262	& 0.5270	& 1.0617	& 1.0575	& 1.0601	& 2.7733	& 2.7702	& 2.7754	& 12.9699	& 12.9703	& 13.0401 \\
OSD$^\text{\tiny\citep{liang2024memory}}$     & 0.5664	& 0.5614	& 0.5645	& 1.0917	& 1.0866	& 1.1003	& 2.7717	& 2.8100	& 2.8067	& 12.8438	& 12.8373	& 12.8594 \\
\textbf{Ori (Ours)} 
        & \best{0.5313}	& \best{0.5260}	& \best{0.5229}	& \best{1.0490}	& \best{1.0473}	& \best{1.0383}	& \best{2.7260}	& \best{2.7018}	& \best{2.6894}	& \best{12.7266}	& 12.7206 & 12.6705 \\
\bottomrule
\multicolumn{13}{c}{\textit{Total Update Time (s)}} \\
\cmidrule(lr){1-13} %
Adam    & 520	& --	& --	& 1032	& --	& --	& 2658	& --	& --	& 12614	& --	& -- \\
ReLoRA$^\text{\tiny\citep{lialin2023relora}}$  & 663	& 640	& 641	& 1334	& 1284	& 1264	& 3338	& 3192	& 3145	& 13931	& \best{11931}	& \best{11578} \\
GaLore$^\text{\tiny\citep{zhao2024galore}}$   & 546	& 527	& 528	& 1063	& 1059	& 1061	& 2776	& 2772	& 2778	& 13053	& 12984	& 12983 \\
OSD$^\text{\tiny\citep{liang2024memory}}$     & 567	& 562	& 565	& 1093	& 1088	& 1101	& 2774	& 2812	& 2809	& 12857	& 12850	& 12872 \\
\textbf{Ori (Ours)} 
        & \best{532}	& \best{526}	& \best{523}	& \best{1049}	& \best{1047}	& \best{1038}	& \best{2728}	& \best{2704}	& \best{2692}	& \best{12740}	& 12733	& 12683 \\
\bottomrule
\end{tabular}
}
\end{table}

\section{Limitations}
\label{sec:limitations}

We discuss several limitations of this work that present opportunities for future research.

\paragraph{Behavior under heterogeneous training data.}
Our empirical evaluation focuses on single-domain pretraining (the C4 dataset~\citep{raffel2020exploring}) and standard fine-tuning benchmarks (the GLUE dataset~\citep{wang2018glue}), following the established experimental protocol used by GaLore~\citep{zhao2024galore}, OSD~\citep{liang2024memory}, and other gradient-subspace methods. An important open question, shared by the many works of low-rank optimizer research, is whether global gradient subspace projection methods perform well when training on highly heterogeneous data mixtures spanning diverse domains (e.g., code, mathematics, chat, and multilingual data simultaneously). In such settings, it is conceivable that the top-$k$ principal gradient directions could be dominated by the most prevalent domains, potentially suppressing rare but important domain-specific gradient signals. Recent work by~\citet{li2025pike} suggests that gradient conflicts are minimal during large-scale LLM pretraining across multilingual and multi-domain settings, which provides initial evidence that this concern may be less severe than expected. Nevertheless, a rigorous investigation---including defining appropriate evaluation metrics for specialized capability retention, curating heterogeneous benchmarks, and systematically comparing low-rank methods under diverse data compositions---would be a valuable contribution to the field. Notably, Ori's high-frequency streaming updates and potential easy composability with standard multi-task optimization methods (e.g., PCGrad~\citep{yu2020gradient}) position it well for such settings, but empirical validation remains future work.

\paragraph{Optimizer state handling across basis changes.}
When Ori updates the gradient projector $P$ via Oja++, the Adam optimizer moments $M$ and $V$ are not explicitly rotated into the new basis. Instead, Ori relies on the slow-drift property ($\rho = 0.99$) and Adam's exponential moving average to naturally decay stale moment information. While this is empirically effective (as validated by the high subspace overlap in \Cref{fig:subspace-overlap} and the strong end-to-end performance), explicitly projecting optimizer states across basis changes---as explored by~\citet{schotthofer2025geometric} in the context of weight-space low-rank factorization---could yield further improvements, particularly in regimes with faster subspace drift or lower $\rho$. We consider this a promising direction for future work.

\paragraph{Adaptive rank.}
Ori operates at a fixed rank $r$ throughout training, which must be specified in advance. Combining Ori's streaming subspace tracking with rank-adaptive methods such as AdaLoRA~\citep{zhang2023adalora}, which dynamically allocates rank budgets across layers based on importance metrics, could enable more flexible and efficient low-rank training. Since Ori's subspace tracking mechanism is orthogonal to rank selection, such a combination is a natural extension.

\section{Broader Social Impact}
\label{app:social_impact}

This paper proposes a novel low-rank optimizer for training large language models, which is a foundational research area in machine learning. As such, we do not foresee any immediate negative societal impacts. However, we recognize that our work, like most AI research, could have broad, long-term implications, including potential issues with energy consumption or, indirectly, in downstream applications.

\end{document}